\documentclass{article}

 \usepackage[preprint]{neurips_2025}

\usepackage[utf8]{inputenc} 
\usepackage[T1]{fontenc}    
\usepackage{hyperref}       
\usepackage{url}            
\usepackage{booktabs}       
\usepackage{amsfonts}       
\usepackage{nicefrac}       
\usepackage{microtype}      
\usepackage{xcolor}         
\usepackage{graphicx}
\usepackage{amsmath, amssymb, amsthm}
\usepackage{algorithm}
\usepackage{algpseudocode}
\theoremstyle{plain}
\newtheorem{theorem}{Theorem}
\usepackage{multirow}
\title{mPOLICE: Provable Enforcement of Multi-Region Affine Constraints in Deep Neural Networks}

\author{%
  Mohammadmehdi Ataei \\
  Autodesk Research \\
  Toronto, Canada \\
  \texttt{mehdi.ataei@autodesk.com} \\
  \And
  Hyunmin Cheong \\
  Autodesk Research \\
  Toronto, Canada \\
  \texttt{hyunmin.cheong@autodesk.com} \\
  \And
  Adrian Butscher \\
  Autodesk Research \\
  Toronto, Canada \\
  \texttt{adrian.butscher@autodesk.com} \\
}

\begin{document}

\maketitle

\begin{abstract}
    Deep neural networks are increasingly used in safety-critical domains such as robotics and scientific modeling, where strict adherence to output constraints is essential. Methods like POLICE, which are tailored for single convex regions, face challenges when extended to multiple disjoint regions, often leading to constraint violations or unwanted affine behavior across regions. This paper proposes mPOLICE, a new approach that generalizes POLICE to provably enforce affine constraints over multiple disjoint convex regions. At its core, mPOLICE assigns distinct neuron activation patterns to each constrained region, enabling localized affine behavior and avoiding unintended generalization. This is implemented through a layer-wise optimization of the network parameters. Additionally, we introduce a training algorithm that incorporates mPOLICE into conventional deep learning pipelines, balancing task-specific performance with constraint enforcement using periodic sign pattern enforcement. We validate the flexibility and effectiveness of mPOLICE through experiments across various applications, including safety-critical reinforcement learning, implicit 3D shape representation with geometric constraints, and fluid dynamics simulations with boundary condition enforcement. Importantly, mPOLICE incurs no runtime overhead during inference, making it a practical and reliable solution for constraint handling in deep neural networks.
\end{abstract}

\section{Introduction}
Deep neural networks (DNNs) have achieved remarkable success in a wide range of domains, from computer vision and natural language processing to scientific simulations and decision-making tasks. However, many real-world applications require these models to produce outputs that satisfy strict constraints. Such constraints often arise from domain knowledge, safety requirements, physical laws, or regulatory guidelines. For example, in climate modeling and fluid simulations, boundary conditions must hold to ensure physically plausible predictions~\citep{PhysRevLett.126.098302,XIE2024117223}; and in robotics, guaranteeing feasible, collision-free trajectories is critical for safety~\citep{kondo2024cgd,bouvier2024policed,bouvier2024learning}.

However, enforcing hard constraints within DNNs is difficult. Traditional training approaches and architectures do not guarantee that constraints will be satisfied, often relying on soft penalties, data augmentation, or post-processing techniques that do not offer any provable guarantees~\citep{kotary2021end,kotary2024learning}. Moreover, strategies that rely on sampling-based corrections or complicated architectures can degrade performance and robustness, or fail to scale efficiently to high-dimensional spaces and complex constraints~\citep{Li_2018_ECCV,tordesillas2023rayen}.

The POLICE algorithm~\citep{police} offers provably optimal affine constraint enforcement in DNNs within a single convex input region by adjusting network biases, without inference overhead or sacrificing expressiveness outside that region. It has been successfully applied in reinforcement learning for safety guarantees~\citep{bouvier2024policed,bouvier2024learning}. However, POLICE is limited to one convex region; naive extension to multiple regions can cause conflicts or unintended affine behavior over their convex hull (see Figure \ref{fig:convexhull}). This limitation has been noted as a key challenge for future research in both the POLICE paper and in follow-up works in robotics and RL~\citep{bouvier2024policed,bouvier2024learning,police}.

In this paper, we present a novel extension of POLICE, referred to as mPOLICE, that overcomes this limitation and enables the exact enforcement of affine constraints in multiple disjoint convex regions simultaneously. Our key insight is to assign unique activation patterns to each constrained region. By doing so, we ensure that each region is distinguished in the network's internal representation, preventing unwanted affine behavior across combined regions. We build on the rigorous theoretical foundation provided by the original POLICE framework~\citep{police}, and integrate recent advances in constrained optimization with deep learning~\citep{kotary2021end,kotary2024learning,PhysRevLett.126.098302,Li_2018_ECCV,tordesillas2023rayen,bouvier2024learning,zhong2023neural} to achieve robust and reliable constraint enforcement.

\textbf{Our contributions can be summarized as follows:}
\begin{itemize}
    \item We introduce mPOLICE, featuring an algorithm to assign unique neuron activation patterns for each constrained region, accompanied by a theoretical guarantee for provably localized affine enforcement.
    \item We provide a robust algorithm (Algorithm~\ref{alg:fine_tuning_mpolice}) to integrate mPOLICE into standard deep learning training for both equality and inequality constraint enforcement.
    \item We demonstrate mPOLICE's versatility across diverse applications, including safety-critical reinforcement learning, implicit 3D shape learning with geometric constraints, enforcing boundary conditions in fluid dynamics, regression, classification, and enforcement in non-convex regions via approximation.
\end{itemize}

\begin{figure*}[ht]
    \centering
    \includegraphics[width=0.9\textwidth]{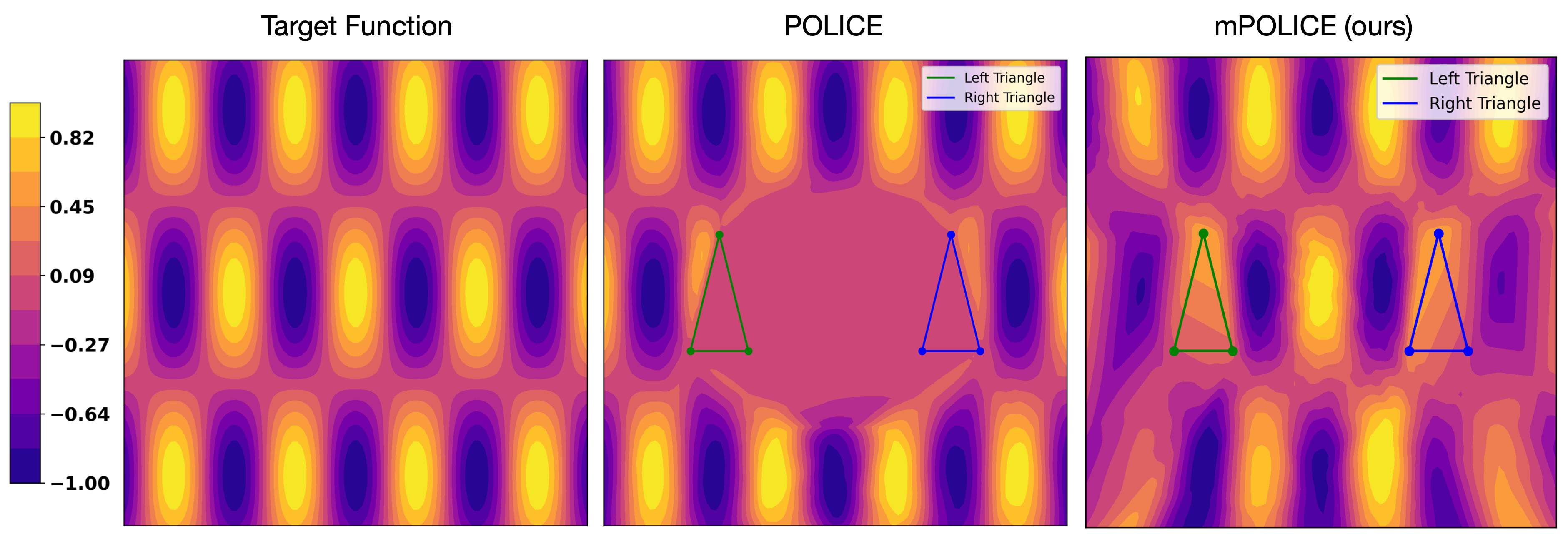}
    \caption{POLICE (single-region) vs. mPOLICE (multi-region) enforcement. POLICE uses one activation pattern, making the network affine over the combined convex hull of all regions (middle). mPOLICE assigns unique patterns to each region, localizing affine behavior and preventing unintended affine behavior across regions (right).}
    \label{fig:convexhull}
\end{figure*}

\subsection{Related Work}
Integrating constraints into neural networks has a history spanning various contexts. Early work includes using neural networks with penalty methods for analog circuit optimization~\citep{286888,995659}, and foundational research in applied dynamic programming connected neural representations to optimization~\citep{bellman2015applied}. More recently, combinatorial optimization for complex constrained problems~\citep{kotary2021end,kotary2024learning}, has also utilized generative models to guide the optimization process~\citep{giannone2023aligning, picard2024generative}, addressing challenges involving physical or domain-specific constraints~\citep{PhysRevLett.126.098302,doi:10.1137/21M1397908,XIE2024117223,pmlr-v168-djeumou22a}.

Beyond penalty methods, direct architectural enforcement of constraints has emerged. This includes ensuring monotonicity, convexity, or linear output constraints~\citep{Li_2018_ECCV,tordesillas2023rayen,konstantinov2024imposing,zhong2023neural}. Physics-informed neural networks (PINNs) embed physical differential constraints into training~\citep{NEURIPS2021_df438e52,NEURIPS2021_d5ade38a}, while other methods use affine or inequality constraints for safe predictions~\citep{kondo2024cgd,bouvier2024policed,bouvier2024learning}. For Bayesian optimization, transformer-based models predict the expected improvements for constraints \citep{yu2024fast}, leveraging their ability for Bayesian inference \citep{muller2021transformers}.

A significant body of work focuses on modifying neural network parameters to ensure their outputs satisfy specific properties, typically by constraining the output to lie within a target set or polytope for given input regions. While methods like those by \citet{tao2023architecture,tao2024provable} provably ensure network outputs lie within target sets for given input regions, our work, mPOLICE, and its predecessor POLICE, uniquely focus on making the network function itself exactly affine (linear) within multiple disjoint input regions by explicitly controlling neuron activation patterns. This distinction in objective (local affine functional form vs. output value containment) and mechanism (activation control vs. output bounding) sets our approach apart, enabling precise local linear behavior.

\section{Methodology}
\label{sec:methodology}

\subsection{Piecewise Affine Structure of ReLU Networks}

A feedforward ReLU network defines a continuous piecewise affine function. Formally, each layer $\ell$ computes

\begin{equation}
\boldsymbol{z}^{(\ell)}=\boldsymbol{W}^{(\ell)}\boldsymbol{x}^{(\ell)}+\boldsymbol{b}^{(\ell)}, \quad \boldsymbol{x}^{(\ell+1)}=\sigma(\boldsymbol{z}^{(\ell)}),
\end{equation}

where $\boldsymbol{W}$ and $\boldsymbol{b}$ are the layer $\ell$'s weights and biases, with $\boldsymbol{x}^{(1)}=\boldsymbol{x}$ where $\boldsymbol{x}$ is input to the network and $\sigma(u)=\max(u,0)$. Each ReLU neuron introduces half-space constraints splitting the input domain into two regions depending on its sign. Stacking $L$ layers yields a finite set of simultaneously satisfiable inequalities that produce a finite collection of $T$ convex polytopes $\{\mathcal{R}_r\}_{r=1}^T$~\citep{montufar2014number}. On each such polytope, the activation pattern is fixed, making $f_{\boldsymbol{\theta}}(\boldsymbol{x})$ an affine function $\boldsymbol{A}_r \boldsymbol{x} + \boldsymbol{c}_r$. This piecewise affine property is central: ensuring each region $R_i$ lies entirely within one such polytope guarantees that $f_{\boldsymbol{\theta}}$ is affine on $R_i$. Note that the same property is applicable to any network with linear or piecewise linear activations (e.g., Leaky-ReLU).

\subsection{Problem Setup and Preliminaries}
\label{subsec:problem_setup}
Consider a deep neural network $f_{\boldsymbol{\theta}}:\mathbb{R}^D \to \mathbb{R}^K$ with parameters $\boldsymbol{\theta}$. Assume there are $N$ disjoint convex polytopal regions $\{R_i\}_{i=1}^N$. Each region $R_i$ can be described by a finite set of vertices $\{\boldsymbol{v}_p^{(i)}\}_{p=1}^{P_i}$. Enforcing constraints in these regions involves ensuring that  $\forall \boldsymbol{x}\in R_i$, the network outputs $f_{\boldsymbol{\theta}}(\boldsymbol{x})$ satisfy certain linear conditions, such as
\begin{align}
    &\boldsymbol{E}_i f_{\boldsymbol{\theta}}(\boldsymbol{x}) = \boldsymbol{f}_i.\label{eqn:equality} \\
    &\boldsymbol{C}_i f_{\boldsymbol{\theta}}(\boldsymbol{x}) \leq \boldsymbol{d}_i,\label{eqn:inequality}
\end{align}
These combined constraints can encode important domain knowledge. The key difficulty is that simply sampling $f_{\boldsymbol{\theta}}$ cannot guarantee constraint satisfaction at all $\boldsymbol{x}\in R_i$.

A solution can be to impose \emph{affinity} over $R_i$ by restricting each region $R_i$ to a unique affine polytope $\mathcal{R}_r$. Then, $f_{\boldsymbol{\theta}}$ becomes a linear function on $R_i$:
\begin{equation}
    f_{\boldsymbol{\theta}}(\boldsymbol{x}) \;=\;
    \mathbf{\Lambda}_i \,\boldsymbol{x} \;+\; \boldsymbol{\gamma}_i, 
    \quad \boldsymbol{x} \in R_i,
\end{equation}

An affine constraint over $R_i$ (e.g., Equations \ref{eqn:equality} or \ref{eqn:inequality}) then only needs to be checked on the \emph{finite} set of vertices $\{\mathbf{v}_p^{(i)}\}_{p=1}^{P_i}$. This reduces the infinite-dimensional verification to a finite set of linear equations or inequalities:
\begin{equation}
\label{eq:vertex_constraints}
    \boldsymbol{E}_i\,(\mathbf{\Lambda}_i\,\mathbf{v}_p^{(i)} + \boldsymbol{\gamma}_i)
    \;=\;
    \boldsymbol{f}_i,
    \quad
    \text{or}
    \quad
    \boldsymbol{C}_i\,(\mathbf{\Lambda}_i\,\mathbf{v}_p^{(i)} + \boldsymbol{\gamma}_i)
    \;\le\;
    \boldsymbol{d}_i.
\end{equation}

However, standard neural network training offers no guarantee that $f_{\boldsymbol{\theta}}$ will be affine on any specific $R_i$, nor that $R_i$ will naturally align with a single affine polytope of the piecewise affine decomposition induced by the network's (Leaky)-ReLU activations. This inherent lack of alignment presents the primary challenge.

\subsection{From Single to Multiple Regions and the Convex Hull Problem}
\label{subsec:from_single_to_multi_regions}

The original POLICE algorithm \citep{police} was designed to ensure the exact affine behavior of a deep ReLU network $f_{\boldsymbol{\theta}}$ within a single convex region $R$. By enforcing consistent pre-activation sign patterns across all vertices of that region, the algorithm guarantees that $R$ is contained within a single activation polytope of the network's piecewise affine decomposition (this is a known property of such networks~\citep{montufar2014number}). At a high level, given $R = \{\boldsymbol{v}_1, \ldots, \boldsymbol{v}_P\}$, the algorithm identifies a binary sign pattern $\boldsymbol{s} = (s_1, \ldots, s_{N_\ell})$ for the $N_\ell$ neurons in layer $\ell$, and adjusts the parameters so that:
\begin{equation}
\label{eq:police_single_region_constraint}
0 \leq \min_{p \in [P]} (\boldsymbol{H}_{p,k}^{(\ell)} s_k), \quad \text{for all } k \in \{1, \ldots, N_\ell\} \,,
\end{equation}
where $\boldsymbol{H}^{(\ell)} \triangleq \boldsymbol{V}^{(\ell)}(\boldsymbol{W}^{(\ell)})^T + \mathbf{1}_{P}(\boldsymbol{b}^{(\ell)})^{T}$ is the pre-activation matrix of layer $\ell$ over the vertices of $R$. Here, $s_k \in \{-1,+1\}$ encodes on which side of the hyperplane defined by the $k$-th neuron (of layer $\ell$) the region $R$ is placed. By ensuring that all vertices share the same sign pattern, $R$ is effectively ``trapped'' inside a single affine polytope of the network. As a result, $f_{\boldsymbol{\theta}}$ behaves as a linear (affine) function on $R$. However, this approach implicitly assumes a single region. Extending it to multiple disjoint regions $\{R_i\}_{i=1}^{N}$ by applying its logic independently to each leads to a problem: multiple regions may be assigned the same or compatible sign patterns. If regions share an activation pattern, the network becomes affine not just on individual regions, but on their entire convex hull—the \emph{convex hull problem}.

This issue stems from POLICE's minimum operation in Equation \ref{eq:police_single_region_constraint}. By seeking a non-negative minimum pre-activation across all vertices in a set, it finds a single activation pattern for them, effectively creating one encompassing affine polytope. When applied to multiple regions, this either merges them into one polytope (if vertices are combined) or only constrains the last region processed (if applied sequentially).

Our work overcomes this by assigning a unique global activation sign pattern to each distinct convex region $R_i$. This assignment \textbf{provably} guarantees that each $R_i$ is confined to its own affine polytope, distinct from others (see Theorem 1 in the technical appendix). This ensures localized affine behavior on each $R_i$, preventing unwanted affine extension over the convex hull of multiple regions and maintaining local constraint integrity.

\subsection{Problem Formulation: Multi-Region Sign Assignment and Parameter Adjustments}
\label{subsec:problem_formulation_multi_region}

We now formulate the general problem of assigning unique sign patterns to multiple disjoint convex regions and adjusting the network parameters accordingly. Suppose we have a feedforward ReLU network $f_{\boldsymbol{\theta}}$ of depth $L$ with parameters $\boldsymbol{\theta}$ (the weights and biases). Let $\{R_i\}_{i=1}^N$ be the set of $N$ disjoint convex regions, each described by its vertices $\boldsymbol{v}_p^{(i)}$. We wish to ensure that each region $R_i$ is contained in a distinct affine polytope of the piecewise affine decomposition induced by the network.

Concretely, we introduce sign variables
\begin{equation}  
    \text{sign}_{n}^{(i,\ell)} \;\in\; \{+1,\,-1\},
\end{equation}

where $\ell \in \{1,\dots,L-1\}$ indexes the layer and $n$ indexes the neuron in layer $\ell$. The sign variable $\text{sign}_{n}^{(i,\ell)}$ encodes that region $R_i$ is placed entirely in the half-space defined by
\begin{equation}
\text{sign}_{n}^{(i,\ell)} 
\left(
    \mathbf{w}_{n}^{(\ell)\top}\,\mathbf{v}_p^{(i,\ell)} 
    + 
    b_n^{(\ell)}
\right) 
\ge \delta,
\quad
\forall\,p \in \{1,\dots,P_i\}
\end{equation}
with $\mathbf{v}_p^{(i,\ell)}$ denoting the vertices after passing through $\ell - 1$ layers and $\delta \ge 0$ a small margin. To force each $R_i$ into a \emph{unique} activation polytope, no two regions may share the same global sign pattern across all neurons and layers.

Formulating these requirements as constraints, we can define the following non-convex optimization problem:
\begin{align*}
&\min_{\{\mathbf{w}_n^{(\ell)},\,b_n^{(\ell)},\,\text{sign}_{n}^{(i,\ell)}\}}
\quad
\Phi\bigl(\boldsymbol{\theta}\bigr)
\quad
\text{subject to}\\
&\quad \text{sign}_{n}^{(i,\ell)}
    \Bigl(
        \mathbf{w}_{n}^{(\ell)\top}\,\mathbf{v}_p^{(i,\ell)} 
        + b_{n}^{(\ell)}
    \Bigr)
    \;\ge\;
    \delta,
    \;\;
    \forall\,p,i,n,\ell,\\
&\quad \exists\, n,\ell \;\text{such that}\ \;
    \text{sign}_{n}^{(i,\ell)}
    \;\neq\;
    \text{sign}_{n}^{(j,\ell)},
    \;\;
    \forall\,i \neq j.
\end{align*}
where $\Phi(\boldsymbol{\theta})$ is an objective function reflecting the primary learning task plus regularization terms. The two sets of constraints can be described as \emph{region-consistency constraints} and \emph{uniqueness constraints}, respectively. The former enforces that for each region $R_i$, all its vertices consistently adhere to the assigned sign pattern $\text{sign}_{n}^{(i,\ell)}$ for each neuron $(n,\ell)$, while the latter enforces that no two distinct regions share the exact same global sign pattern across all neurons and layers.

Solving this problem, which takes the form of a mixed-integer problem assuming the sign variables are binary, is NP-hard. Hence, in practice, we can employ heuristics to determine $\text{sign}_{n}^{(i,\ell)}$ first, and then solve simpler sub-problems (e.g., quadratic or linear programs) to enforce the assigned half-space constraints by adjusting $\{\mathbf{w}_n^{(\ell)}, b_n^{(\ell)}\}$ at each layer separately.

\subsection{Strategies for Sign Assignment}
\label{subsec:assigning_unique_patterns}
We propose two heuristic methods for determining each region's signs:

\paragraph{Majority Voting.} 
For each region $R_i$, we examine its vertex pre-activations $\{z_n^{(\ell)}(\mathbf{v}_p^{(i)})\}$ at layer $\ell$. We then set $\text{sign}_n^{(i,\ell)} = +1$ if the majority of $\{z_n^{(\ell)}(\mathbf{v}_p^{(i)})\}$ are positive; otherwise, we choose $-1$. Zeros are treated as positive if they appear. This is a simple, low-cost strategy and often provides reliable region separation, especially when each neuron has a clear tendency to be either positive or negative over $R_i$.
    
\paragraph{Pre-Activation Mean-based.} For each region $R_i$ and neuron $n$ in layer $\ell$, we compute the average of the pre-activations over the vertices of $R_i$. Specifically, let
\begin{align*}
m_{n}^{(i,\ell)} &= \frac{1}{P_i} \sum_{p=1}^{P_i} z_n^{(\ell)}\bigl(\mathbf{v}_p^{(i)}\bigr)
\end{align*}
where $z_n^{(\ell)}(\boldsymbol{x}) = \mathbf{w}_n^{(\ell)\top}\boldsymbol{x} + b_n^{(\ell)}$. We then set $\text{sign}_{n}^{(i,\ell)} = +1$ if $m_{n}^{(i,\ell)} \ge 0$ and $-1$ otherwise. When $m_{n}^{(i,\ell)}$ is extremely close to zero, we may impose a small margin to avoid sign ambiguity.

Selecting the right approach depends on how pre-activations distribute across vertices. The mean-based method works well when they cluster around distinct positive or negative values, making outliers less influential. Majority voting is more suitable if nearly all vertices share the same sign. It is robust to small sets of outliers but can become unstable if the region straddles the boundary, where a near-even split may flip the result.

After assigning initial sign patterns $\{\text{sign}_n^{(i,\ell)}\}$ to each region $R_i$, we ensure no two regions share the same global pattern. If duplicates occur, often due to nearby constrained regions, we flip the signs of a few neurons (with near-zero pre-activations) in one layer of one region to guarantee uniqueness.

\subsection{Enforcing Sign Patterns}
\label{subsec:enforcing_signs}
Assigned sign patterns dictate target polytopes but don't guarantee parameter compliance. A bias-only adjustment strategy, as used in POLICE \citep{police}, proves insufficient for multiple disjoint regions (see proof in the technical appendix). Thus, we must adjust both weights and biases to ensure each region $R_i$ lies in a distinct polytope. To solve this issue, we can solve a small quadratic (or linear) program to fine-tune both $\mathbf{w}_n^{(\ell)}$ and $b_n^{(\ell)}$ with minimal parameter shifts. The quadratic objective (minimizing the L2 norm of parameter changes) is often preferred as it encourages smaller, more distributed adjustments, which can be less disruptive to the learned representations compared to an L1 objective that might induce sparser but larger changes. Concretely, to enforce the assigned signs, for each neuron $n$ in layer $\ell$, we find the minimal parameter adjustments $\Delta \mathbf{w}_n^{(\ell)}$ and $\Delta b_n^{(\ell)}$ to its current parameters (denoted $\mathbf{w}_n^{(\ell)}$ and $b_n^{(\ell)}$ respectively) by solving:
\begin{align*}
    &\min_{\Delta \mathbf{w}_n^{(\ell)},\, \Delta b_n^{(\ell)}}
    \quad
    \|\Delta \mathbf{w}_n^{(\ell)}\|^2 + (\Delta b_n^{(\ell)})^2 \\
    &\text{subject to} \\
    &\quad \text{sign}_n^{(i,\ell)} \,
    \Bigl(
      (\mathbf{w}_n^{(\ell)} + \Delta\mathbf{w}_n^{(\ell)})^\top \mathbf{v}_p^{(i,\ell)}
      + (b_n^{(\ell)} + \Delta b_n^{(\ell)})
    \Bigr) \,\ge\, \delta,
    \quad \forall p, i.
\end{align*}
This yields a minimal-norm update to each neuron's parameters that enforces the assigned signs exactly.

\subsection{Algorithm for Imposing Affine Constraints during Training}
\label{subsec:imposingconst}
Once the strategies from Section~\ref{subsec:assigning_unique_patterns} have assigned unique sign patterns to each region $R_i$, and the methods from Section~\ref{subsec:enforcing_signs} have adjusted network parameters to ensure $R_i$ lies within a unique polytope, the network $f_{\boldsymbol{\theta}}$ behaves as an affine function in $R_i$ having the form $\boldsymbol{\Lambda}_i \boldsymbol{x} + \boldsymbol{\gamma}_i$ for some $\boldsymbol{\Lambda}_i \in \mathbb{R}^{K \times D}$ and $\boldsymbol{\gamma}_i \in \mathbb{R}^K$. Consequently, as noted in Section~\ref{subsec:problem_setup}, the affine constraints (Equations~\ref{eqn:equality} and~\ref{eqn:inequality}) only need to be satisfied at the vertices $\{\boldsymbol{v}_p^{(i)}\}_{p=1}^{P_i}$ of each region.

To achieve this during training, we employ an iterative fine-tuning process, detailed in Algorithm~\ref{alg:fine_tuning_mpolice}. This procedure alternates between optimizing a composite loss function and re-enforcing the predetermined sign patterns. The composite loss is defined as:
\begin{align*}
\mathcal{L}_{\text{total}} = \mathcal{L}_{\text{task}} + \lambda_{\text{constraint}} \mathcal{L}_{\text{constraint}}
\end{align*}
where $\mathcal{L}_{\text{task}}$ is the primary loss for the learning task (e.g., cross-entropy or mean squared error), and $\mathcal{L}_{\text{constraint}}$ penalizes violations of the affine constraints at the vertices. Specifically, for equality constraints $\boldsymbol{E}_i f_{\boldsymbol{\theta}}(\boldsymbol{v}_p^{(i)}) = \boldsymbol{f}_i$, the penalty is $\sum_{i,p} \|\boldsymbol{E}_i f_{\boldsymbol{\theta}}(\boldsymbol{v}_p^{(i)}) - \boldsymbol{f}_i\|^2$. For inequality constraints $\boldsymbol{C}_i f_{\boldsymbol{\theta}}(\boldsymbol{v}_p^{(i)}) \le \boldsymbol{d}_i$, the penalty is $\sum_{i,p} \|\max(\boldsymbol{0}, \boldsymbol{C}_i f_{\boldsymbol{\theta}}(\boldsymbol{v}_p^{(i)}) - \boldsymbol{d}_i)\|^2$. The factor $\lambda_{\text{constraint}}$ is a dynamically adjusted weight.

The constraint enforcement (Algorithm~\ref{alg:fine_tuning_mpolice}) typically follows an initial training phase where the network learns the primary task's overall patterns, establishing a base model. The subsequent fine-tuning stage integrates the affine constraints. It starts by assigning and enforcing a unique sign map $\mathcal{M}_{\text{signs}}$ for all constrained regions (Section~\ref{subsec:assigning_unique_patterns}). Each fine-tuning epoch then performs mini-batch gradient descent on a composite loss $\mathcal{L}_{\text{total}}$. Crucially, this $\mathcal{M}_{\text{signs}}$ is re-enforced after each epoch's parameter updates (Section~\ref{subsec:enforcing_signs}) to counteract optimizer drift and ensure each $R_i$ stays within its designated affine polytope, preserving multi-region affinity. The penalty $\lambda_{\text{constraint}}$ is increased if violation $V$ persists above $\epsilon$. Optionally, data within equality-constrained regions can be filtered from $\mathcal{L}_{\text{task}}$ to prevent conflicting signals.

\begin{algorithm}[ht]
\caption{mPOLICE Training for Multi-Region Constraint Enforcement}
\label{alg:fine_tuning_mpolice}
\begin{algorithmic}[1]
\Require Network $f_{\boldsymbol{\theta}}$ (pre-trained base model), regions $\{R_i\}_{i=1}^N$ with vertices, associated constraints $\{\boldsymbol{E}_i, \boldsymbol{f}_i, \boldsymbol{C}_i, \boldsymbol{d}_i\}$, a fixed sign map $\mathcal{M}_{\text{signs}}$ assigning unique patterns to each $R_i$, training data $(\boldsymbol{X}_{\text{train}}, \boldsymbol{Y}_{\text{train}})$, task loss $\mathcal{L}_{\text{task}}$, optimizer, initial $\lambda_{\text{constraint}}$, tolerance $\epsilon$.
\State \texttt{// Initial enforcement of sign patterns}
\State Enforce sign patterns on $f_{\boldsymbol{\theta}}$ based on $\mathcal{M}_{\text{signs}}$ using method from Sec.~\ref{subsec:enforcing_signs}.
\State Initialize patience counter $p_c \gets 0$, best overall loss $L_{\text{best}} \gets \infty$.
\For{epoch = 1 to max\_epochs}
    \State \texttt{// Optimize task and constraint objectives}
    \For{each batch $(\boldsymbol{x}_{\text{batch}}, \boldsymbol{y}_{\text{batch}})$ from $(\boldsymbol{X}_{\text{train}}, \boldsymbol{Y}_{\text{train}})$}
        \State Compute task loss $\mathcal{L}_{\text{task}}(f_{\boldsymbol{\theta}}(\boldsymbol{x}_{\text{batch}}), \boldsymbol{y}_{\text{batch}})$.
        \State Initialize constraint penalty $\mathcal{L}_{\text{constraint}} \gets 0$.
        \For{each region $R_i$ with vertices $\{\boldsymbol{v}_p^{(i)}\}$}
            \State Evaluate $f_{\boldsymbol{\theta}}(\boldsymbol{v}_p^{(i)})$ at all vertices of $R_i$.
            \State Add penalties to $\mathcal{L}_{\text{constraint}}$ for violated equality/inequality constraints at vertices.
        \EndFor
        \State $\mathcal{L}_{\text{total}} \gets \mathcal{L}_{\text{task}} + \lambda_{\text{constraint}} \mathcal{L}_{\text{constraint}}$.
        \State Perform optimizer step on $\mathcal{L}_{\text{total}}$ to update $f_{\boldsymbol{\theta}}$.
    \EndFor
    \State \texttt{// Re-enforce sign patterns after parameter updates}
    \State Enforce sign patterns on $f_{\boldsymbol{\theta}}$ based on $\mathcal{M}_{\text{signs}}$ using method from Sec.~\ref{subsec:enforcing_signs}.
    \State Measure current aggregate constraint violation $V$ across all $R_i$.
    \State Compute a balanced loss $L_{\text{balanced}}$ (e.g., $\sqrt{\text{avg}(\mathcal{L}_{\text{task}}) \cdot (1+V)}$).
    \If{$L_{\text{balanced}} < L_{\text{best}}$}
        \State $L_{\text{best}} \gets L_{\text{balanced}}$, save current $f_{\boldsymbol{\theta}}$, $p_c \gets 0$.
    \Else
        \State Increment $p_c$.
    \EndIf
    \If{$p_c \ge \text{patience\_threshold}$}
        \If{$V > \epsilon$ and $\lambda_{\text{constraint}} < \lambda_{\text{max}}$}
            \State Increase $\lambda_{\text{constraint}}$ (e.g., $\lambda_{\text{constraint}} \gets \lambda_{\text{constraint}} \cdot \text{multiplier}$).
            \State $p_c \gets 0$.
        \Else
            \State \textbf{break} \texttt{// Converged by patience or constraint met/max penalty}
        \EndIf
    \EndIf
    \If{$V \le \epsilon$ and epoch $\ge$ min\_epochs}
        \State \textbf{break} \texttt{// Converged by meeting tolerance}
    \EndIf
\EndFor
\State Restore best saved $f_{\boldsymbol{\theta}}$.
\end{algorithmic}
\end{algorithm}

\section{Experiments}

In addition to the experiments presented in this section, further examples, including classification and regression tasks, are detailed in the supplementary material\footnote{The full code will be made available after publication.}. All experiments were run on an NVIDIA RTX 6000 ADA GPU with 144\,GB RAM and a 16-core AMD 5955WX CPU. Results are averaged over 5 random seeds, with error bars showing $\pm1$ standard deviation ($\sigma$).

\subsection{Reinforcement Learning}
\label{subsec:rl}

In this section, we demonstrate mPOLICE's application to a reinforcement learning (RL) task, enforcing safety-critical constraints on an agent's policy. The agent, trained with the Twin Delayed Deep Deterministic Policy Gradient (TD3) algorithm~\citep{fujimoto2018addressing}, navigates a 2D environment to reach a target while avoiding two static obstacles. This extends prior work~\citep{bouvier2024policed} that used the original POLICE algorithm for a single constrained buffer region, by now handling multiple such regions. mPOLICE constrains the agent's policy $\pi_{\boldsymbol{\theta}}(\boldsymbol{s})$ within buffer zones $R_i$ to ensure safe actions (e.g., an action component $a_y \le 0$). By assigning unique activation patterns in each $R_i$, mPOLICE renders the policy locally affine: $\pi_{\boldsymbol{\theta}}(\boldsymbol{s}) = \mathbf{\Lambda}_i \boldsymbol{s} + \boldsymbol{\gamma}_i$ for $\boldsymbol{s} \in R_i$. The safety condition, an affine constraint on the output (e.g., $[0,1]\pi_{\boldsymbol{\theta}}(\boldsymbol{s}) \le 0$), thus becomes an affine constraint on this local function, verifiable by checking only at the vertices of $R_i$. The supplementary material details the RL setup and specific constraints.

Figure~\ref{fig:rl_results} (right) shows the training dynamics. mPOLICE enforcement begins around 16,000 timesteps, causing brief spikes in critic and actor losses as it adjusts network parameters to enforce affine behavior and constraints. These spikes are typical in constrained optimization. Despite them, evaluation reward steadily rises due to safer, more successful policies. The left panel shows the learned policy and affine partitions, visualized using adapted techniques from~\citep{humayun2022splinecam, humayun2022exact}. Blue arrows indicate learned actions, and black lines mark the policy network’s affine polytope boundaries. The inset shows each red buffer region lies entirely within a single polytope, confirming that mPOLICE enforces safe, affine behavior within these regions.

\begin{figure}[ht]
    \centering
    \includegraphics[width=\textwidth]{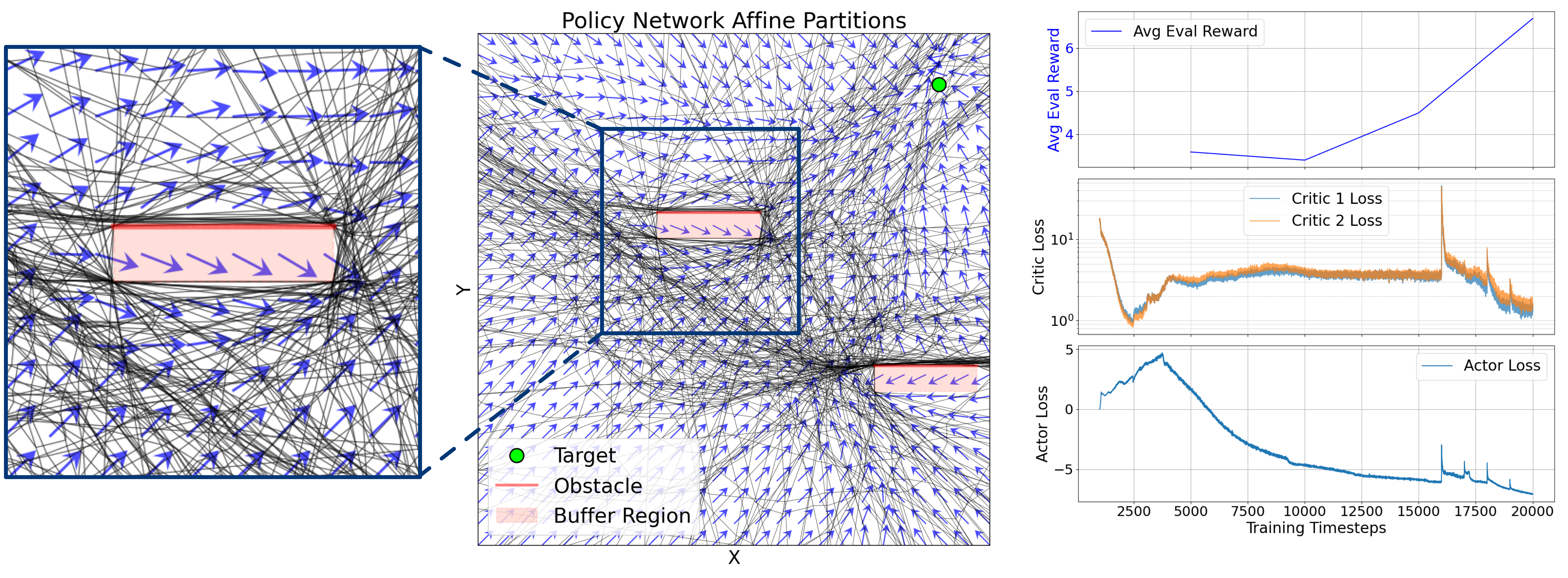}
    \caption{Reinforcement learning results. Left: Policy network affine partitions with learned policy (blue arrows), target (green circle), obstacles (red), buffer regions (light pink), and affine polytope boundaries from ReLU activations (thin black lines). Right: Training curves showing average evaluation reward, critic losses, and actor loss.}
    \label{fig:rl_results}
\end{figure}

\subsection{Implicit Occupancy-Field Learning with Geometric Constraints}
\label{implicit_shape}
In this experiment, a neural network learns an implicit 3D object representation featuring two cylindrical holes. The network maps 3D coordinates to a scalar indicating if a point is in the main body or empty. Accurate hole representation is vital for applications like 3D modeling. mPOLICE enforces that the network outputs the value of zero (as an equality constraint) for points within the hole regions, ensuring they are empty.

\begin{table}[ht]
    \centering
    \caption{3D shape representation (mean $\pm$ standard deviation over five runs). “Violation’’ is the mean absolute deviation from the target zero value in hole regions. “FineTune\%’’ is mPOLICE fine-tuning time relative to initial training. Threshold ($1\times 10^{-3}$) met for all experiments.}
    \label{tab:shape_results}
    \resizebox{\textwidth}{!}{%
    \begin{tabular}{lccccc}
        \toprule
        Method & MSE Overall & Violation & Initial Train (s) & FineTune (s) & FineTune \% \\
        \midrule
        Baseline & 0.001909 $\pm$ 0.000109 & 40.942 $\pm$ 1.791 & 209.54 $\pm$ 15.78 & 0.00 $\pm$ 0.00 & N/A \\
        mPOLICE & 0.006498 $\pm$ 0.000609 & 0.000350 $\pm$ 0.000048 & 209.54 $\pm$ 15.78 & 112.94 $\pm$ 25.92 & 53.9 $\pm$ 13.0\% \\
        \bottomrule
    \end{tabular}%
    }
\end{table}

The performance of mPOLICE in this task is compared against a baseline model that is trained without explicit constraint enforcement. As shown in Table~\ref{tab:shape_results}, the baseline model's high constraint violation (40.942) indicates failed hole representation. In contrast, mPOLICE reduces this violation to 0.000350, well below the $1.0 \times 10^{-3}$ threshold, successfully meeting the constraint. This improved adherence slightly increases overall MSE, an expected trade-off. The mPOLICE fine-tuning process was approximately 53.9\% of the initial training time.

Figure~\ref{fig:shape_3d_comparison} highlights the qualitative differences. The ground truth and voxelized baseline/mPOLICE predictions are shown. Voxelization discretizes the network's continuous output. The baseline network mischaracterizes some empty spaces as being part of the solid object, whereas mPOLICE accurately represents them, demonstrating its geometric constraint enforcement.

\begin{figure}[ht]
    \centering
    \includegraphics[width=\textwidth]{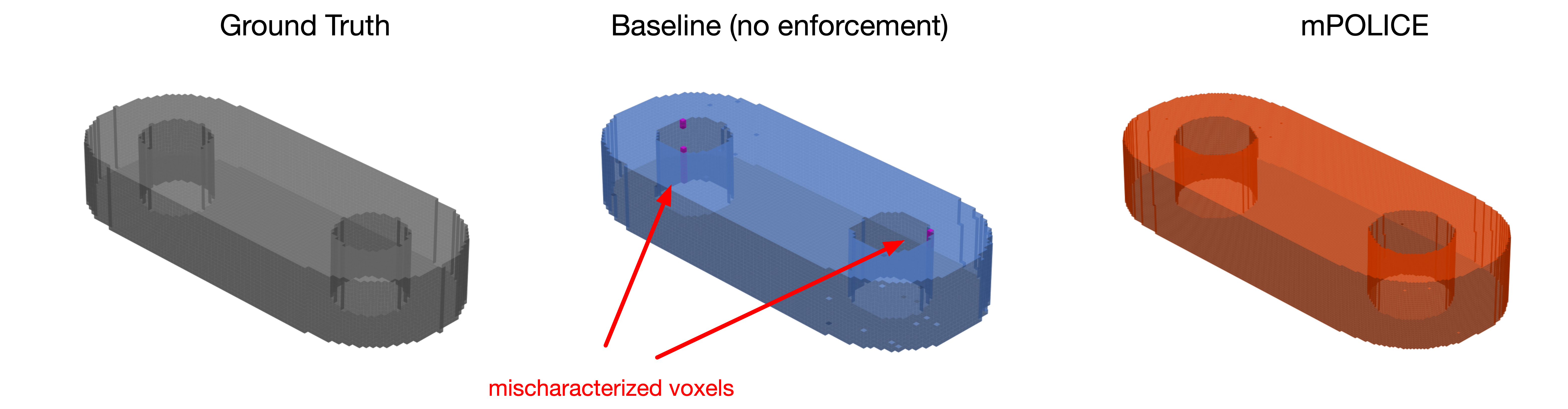}
    \caption{3D shape representations: Ground Truth, Baseline Prediction, and mPOLICE Prediction. The baseline's voxelized output incorrectly fills holes (magenta), while mPOLICE accurately depicts empty cylindrical holes, satisfying constraints.}
    \label{fig:shape_3d_comparison}
\end{figure}

\subsection{Fluid Dynamics}
\label{subsec:fluid_dynamics}

This experiment uses mPOLICE to enforce zero-velocity boundary conditions in a 2D fluid dynamics simulation. The objective is to learn the velocity magnitude field of a fluid flow generated by the XLB~\citep{ataei2024xlb} library, ensuring zero velocity within two disjoint square regions that represent obstacles. A neural network is trained for this task, and we compare a standard (baseline) model with one enhanced by mPOLICE constraints.

The results of this experiment are summarized in Table~\ref{tab:fluid_dynamics_results}, which compares the baseline and constrained models across several metrics. The constrained model achieves near-zero constraint violations and demonstrates slightly higher MSE compared to the baseline. The mPOLICE fine-tuning phase accounts for approximately 19.4\% of the total training time for the constrained model. Experimental details, hyperparameters, and result visualizations are provided in the supplementary material.

\begin{table}[ht]
\centering
\caption{Fluid dynamics example: “Violation” is constraint violation. “FineTune\%” is mPOLICE fine-tuning as \% of total constrained training. Threshold ($1\times 10^{-4}$) met for all experiments.}
\label{tab:fluid_dynamics_results}
\resizebox{\textwidth}{!}{
\begin{tabular}{lccccc}
\toprule
Method & MSE Overall & Violation & Initial Train (s) & FineTune (s) & FineTune \% \\
\midrule
Baseline & $1.40\times10^{-6} \pm 5.5\times10^{-7}$ & $8.46\times10^{-4} \pm 5.13\times10^{-4}$ & $141.69 \pm 2.37$ & $0.00 \pm 0.00$ & N/A \\
mPOLICE  & $2.60\times10^{-6} \pm 8.9\times10^{-7}$ & $2\times10^{-6} \pm 1\times10^{-6}$ & $141.69 \pm 2.37$ & $34.19 \pm 6.70$ & $24.1 \pm 4.7\%$ \\
\bottomrule
\end{tabular}
}
\end{table}

\subsection{Scalability}
\label{subsec:scalability}

The computational cost of mPOLICE arises mainly from the enforcement step ( $T_{\text{Enforce}}$), which solves quadratic programs in each network layer. We have provided detailed scaling experiments in the technical appendix that show $T_{\text{Enforce}}$ increases with network depth, width, and the total number of vertices in constrained regions. Sign assignment time ($T_{\text{Assign}}$ in the technical appendix) is typically negligible. While high-dimensional shapes can have many vertices, practical applications (e.g., robotics, physics) often involve constraints on lower-dimensional subspaces with manageable vertex counts. Thus, solving mPOLICE can still be manageable in many practical scenarios.

Our current enforcement implementation is not highly optimized (e.g., QPs are solved sequentially on CPU, which also adds communication overhead). Note that enforcement time is independent of training dataset size, adding a fixed cost per invocation. Therefore, its percentage of total training time varies significantly with base training duration. In our experiments presented before, mPOLICE's fine-tuning phase (including repeated enforcement) was roughly 20-70\% of the baseline training. Note that mPOLICE has zero inference overhead.

\subsection{Non-Convex Approximation}
\label{sec:non_convex_approximation}
While mPOLICE is primarily designed for convex regions, it can be extended to approximate non-convex ones by decomposing them into a collection of closely-aligned convex sub-regions. Further details and an empirical demonstration are provided in the technical appendix.

\section{Limitations}
\label{sec:limitations}

While mPOLICE offers a provable method for multi-region affine constraint enforcement, it faces limitations. Scalability is a key concern, as the sign pattern enforcement step, involving quadratic programs, can be computationally intensive with larger networks or numerous vertices. Moreover, our current implementation is not optimized for high-performance. The method is limited to MLPs with piecewise linear activations, which, although restrictive, still covers a wide range of applications. Extending it to non-linear activations or other architectures like CNNs and Transformers remains future work. The heuristics for assigning unique sign patterns are practical but not foolproof, and successful separation depends on network capacity. Furthermore, the fine-tuning process introduces hyperparameters that may require tuning. Addressing these aspects is crucial for broader applicability.

\section{Conclusion}
\label{sec:conclusion}

We introduced mPOLICE, an extension of the POLICE algorithm, for enforcing affine constraints in deep neural networks across multiple disjoint convex regions. By assigning unique activation patterns to each region, mPOLICE provably ensures localized affine behavior, overcoming the prior "convex hull problem" and guaranteeing each region lies within a distinct affine polytope. Our practical algorithm integrates this into standard training, demonstrating efficacy in diverse applications like safety-critical reinforcement learning and geometrically constrained 3D modeling, while having zero inference cost. Future work will explore broader applications in complex simulations and Neural Implicit Representations.

\bibliography{main}
\bibliographystyle{plainnat}

\newpage
\appendix
\section{Proof of Localized Affine Behavior with Unique Activation Patterns}
\begin{theorem}[Localized Affine Behavior with Unique Activation Patterns]
    \label{thm:concise_proof}
    Let $f_{\boldsymbol{\theta}}$ be a feedforward ReLU network of depth $L$.
    Consider a convex region $R = \mathrm{conv}\{\boldsymbol{v}_1,\dots,\boldsymbol{v}_P\}$ in the input space, and a set of $N$ disjoint convex regions $\{R_i\}_{i=1}^N$, where each $R_i = \mathrm{conv}\{\boldsymbol{v}_{i,1},\dots,\boldsymbol{v}_{i,P_i}\}$ in the input space.
    \begin{enumerate}
        \item \textbf{Single Region Affinity:} If $R$ has a consistent pre-activation sign pattern across all its vertices for all neurons in all layers (i.e., for each neuron, its pre-activation is either non-negative for all vertices of $R$, or non-positive for all vertices of $R$), then $R$ is contained entirely within a single affine polytope of $f_{\boldsymbol{\theta}}$. Consequently, $f_{\boldsymbol{\theta}}$ is affine on $R$.
        \item \textbf{Multi-Region Localized Affinity:} If each region $R_i$ is assigned a \emph{unique global activation sign pattern} $\boldsymbol{S}_i$ (where $\boldsymbol{S}_i$ is the collection of pre-activation signs for all neurons across all layers, and $\boldsymbol{S}_i \neq \boldsymbol{S}_j$ for $i \neq j$), and this pattern $\boldsymbol{S}_i$ is consistently maintained for all vertices within $R_i$, then each region $R_i$ is contained within its own distinct affine polytope $\mathcal{P}_i$. As a result, $f_{\boldsymbol{\theta}}$ is affine on each $R_i$ individually. This affine behavior is strictly localized to $R_i$, preventing unintended affine behavior over the convex hull of combinations of these regions due to shared activation patterns.
    \end{enumerate}
\end{theorem}

\begin{proof}
\textbf{Part 1: Single Region Affinity.}
We prove by induction on the layer index $\ell$ that if pre-activation signs are consistent for the vertices of $R$ up to layer $\ell$, then the network's mapping up to the output of layer $\ell$'s activations is affine on $R$.

Base Case ($\ell=0$, input layer): The input $\boldsymbol{x}$ is trivially affine on $R$. The pre-activations for the first hidden layer are $z_k^{(1)}(\boldsymbol{x}) = \boldsymbol{w}_k^{(1)\top}\boldsymbol{x} + b_k^{(1)}$. If $z_k^{(1)}(\boldsymbol{v}_p)$ has a consistent sign for all vertices $\boldsymbol{v}_p$ of $R$, then by linearity of $z_k^{(1)}$ and convexity of $R$, $z_k^{(1)}(\boldsymbol{x})$ maintains this sign for all $\boldsymbol{x} \in R$. (Any $\boldsymbol{x} \in R$ is a convex combination $\sum \alpha_p \boldsymbol{v}_p$, so $z_k^{(1)}(\boldsymbol{x}) = \sum \alpha_p z_k^{(1)}(\boldsymbol{v}_p)$, which will have the same sign as $z_k^{(1)}(\boldsymbol{v}_p)$ if they are all consistent). Thus, the ReLU activation $\sigma(z_k^{(1)}(\boldsymbol{x}))$ becomes $z_k^{(1)}(\boldsymbol{x})$ for all $\boldsymbol{x} \in R$ (if $z_k^{(1)}(\boldsymbol{v}_p) \ge 0$ for all vertices $\boldsymbol{v}_p$ of $R$), or it becomes $0$ for all $\boldsymbol{x} \in R$ (if $z_k^{(1)}(\boldsymbol{v}_p) \le 0$ for all $p$).  In both cases, the output of the first layer's activations, $\boldsymbol{x}^{(2)}(\boldsymbol{x})$, is an affine function of $\boldsymbol{x}$ when restricted to $R$.

Inductive Hypothesis: Assume that for layers $1, \dots, \ell-1$, consistent pre-activation signs for the vertices of $R$ imply that the output of layer $\ell-1$'s activations, $\boldsymbol{x}^{(\ell)}(\boldsymbol{x})$, is an affine function of the original input $\boldsymbol{x}$ when restricted to $R$.

Inductive Step (Layer $\ell$): The pre-activations for layer $\ell$ are $z_k^{(\ell)}(\boldsymbol{x}) = \boldsymbol{w}_k^{(\ell)\top}\boldsymbol{x}^{(\ell)}(\boldsymbol{x}) + b_k^{(\ell)}$. Since $\boldsymbol{x}^{(\ell)}(\boldsymbol{x})$ is affine on $R$ (by IH), $z_k^{(\ell)}(\boldsymbol{x})$ is also affine on $R$. By the same argument as the base case, if $z_k^{(\ell)}$ has a consistent sign across all vertices of $R$, it maintains that sign throughout $R$. Therefore, the output of layer $\ell$'s activations, $\boldsymbol{x}^{(\ell+1)}(\boldsymbol{x}) = \sigma(\boldsymbol{z}^{(\ell)}(\boldsymbol{x}))$, is an affine function of $\boldsymbol{x}$ when restricted to $R$.

This holds up to the final layer $L$. If all pre-activation signs are consistent within $R$ for all neurons in all layers, then $R$ lies within a single region where all ReLU activation choices are fixed. This region is an affine polytope, and the network function $f_{\boldsymbol{\theta}}(\boldsymbol{x})$ becomes an affine transformation $\boldsymbol{A}_R \boldsymbol{x} + \boldsymbol{c}_R$ for $\boldsymbol{x} \in R$.

\textbf{Part 2: Multiple Disjoint Regions with Unique Patterns.}
From Part 1, if region $R_i$ maintains its assigned global activation sign pattern $\boldsymbol{S}_i$ consistently across its vertices, it is contained within a single affine polytope, let's call it $\mathcal{P}_i$. This polytope $\mathcal{P}_i$ is defined by the specific set of active/inactive states of all ReLU neurons as dictated by $\boldsymbol{S}_i$.

Now, consider two distinct regions $R_i$ and $R_j$ ($i \neq j$), which are assigned unique global activation patterns $\boldsymbol{S}_i$ and $\boldsymbol{S}_j$ respectively, with $\boldsymbol{S}_i \neq \boldsymbol{S}_j$. The fact that $\boldsymbol{S}_i \neq \boldsymbol{S}_j$ means there exists at least one neuron $k$ in some layer $m$ for which the pre-activation sign for $R_i$ (say, $s_{mk}^{(i)}$) is different from the pre-activation sign for $R_j$ (say, $s_{mk}^{(j)}$).
For $R_i$ to be in $\mathcal{P}_i$, all its points must satisfy the half-space conditions corresponding to $\boldsymbol{S}_i$. Specifically, for neuron $k$ in layer $m$, points in $\mathcal{P}_i$ satisfy $s_{mk}^{(i)} (\boldsymbol{w}_{mk}^{(m)\top}\boldsymbol{x}^{(m)} + b_{mk}^{(m)}) \ge 0$.
For $R_j$ to be in its polytope $\mathcal{P}_j$, all its points must satisfy the half-space conditions corresponding to $\boldsymbol{S}_j$. Specifically, for the same neuron $k$ in layer $m$, points in $\mathcal{P}_j$ satisfy $s_{mk}^{(j)} (\boldsymbol{w}_{mk}^{(m)\top}\boldsymbol{x}^{(m)} + b_{mk}^{(m)}) \ge 0$.

Since $s_{mk}^{(i)} \neq s_{mk}^{(j)}$, it means $s_{mk}^{(i)} = -s_{mk}^{(j)}$. Therefore, the condition defining one side of the hyperplane for neuron $mk$ that $R_i$ must satisfy is precisely the opposite of the condition that $R_j$ must satisfy for that same neuron. This implies that $R_i$ and $R_j$ lie on opposite sides of the hyperplane defined by neuron $mk$ (or one is on the hyperplane if the pre-activation is zero, and the other is strictly on one side). Consequently, $R_i$ and $R_j$ cannot be in the same affine polytope. Each region $R_i$, by virtue of its unique and consistently applied global activation pattern $\boldsymbol{S}_i$, is confined to its own distinct affine polytope $\mathcal{P}_i$.

This ensures that the affine behavior $f_{\boldsymbol{\theta}}(\boldsymbol{x}) = \boldsymbol{A}_i \boldsymbol{x} + \boldsymbol{c}_i$ is specific to $R_i$ and $f_{\boldsymbol{\theta}}(\boldsymbol{x}) = \boldsymbol{A}_j \boldsymbol{x} + \boldsymbol{c}_j$ is specific to $R_j$, and since $\mathcal{P}_i \neq \mathcal{P}_j$, there is no unintended affine relationship enforced across the convex hull of $R_i \cup R_j$ that would arise if they shared the same polytope. The affine properties are localized.
\end{proof} 

\section{Example: Contradiction in the Bias-Only Approach}
\label{app:contradiction_bias_only}

A neuron's bias term shifts its activation function, effectively moving the decision boundary (the hyperplane where its pre-activation is zero). While adjusting only the bias can successfully enforce a consistent sign pattern for a neuron's pre-activation across a single convex input region (as in the original POLICE algorithm), this approach can fail when attempting to enforce \emph{different} or conflicting sign patterns across multiple disjoint regions. This is because a single bias value per neuron may not be flexible enough to satisfy all regional constraints simultaneously if these constraints impose contradictory requirements on the neuron's activation threshold. We illustrate this fundamental limitation with a clear example.

\paragraph{Explanation of the Setup}
Consider a single-layer network with one neuron, whose pre-activation for a scalar input $x$ is $z(x) = w \cdot x + b$. Assume the weight $w$ is fixed and positive ($w > 0$). We define two disjoint convex regions: $R_1 = [0, 1]$ and $R_2 = [2, 3]$.
Our goal is to find a single bias $b$ such that:
\begin{enumerate}
    \item For $R_1$: $z(x) \ge 0$ (neuron pre-activation is non-negative).
    \item For $R_2$: $z(x) \le 0$ (neuron pre-activation is non-positive).
\end{enumerate}
This requires $\min_{x \in R_1} (w \cdot x + b) \ge 0$ and $\max_{x \in R_2} (w \cdot x + b) \le 0$.

\paragraph{Proof of Contradiction}
Since $w > 0$, $z(x)$ is monotonically increasing.
\begin{enumerate}
    \item \textbf{For $R_1 = [0, 1]$}, the minimum of $z(x)$ occurs at $x=0$. So, $w \cdot 0 + b \ge 0 \implies b \ge 0$.
    \item \textbf{For $R_2 = [2, 3]$}, the maximum of $z(x)$ occurs at $x=3$. So, $w \cdot 3 + b \le 0 \implies b \le -3w$.
\end{enumerate}
We have $b \ge 0$ and $b \le -3w$. Since $w > 0$, $-3w$ is a strictly negative number. It is impossible for $b$ to be simultaneously non-negative and less than or equal to a strictly negative value. This is a contradiction.

\paragraph{Conclusion}
This example demonstrates that a bias-only adjustment strategy cannot, in general, reconcile conflicting sign requirements across multiple disjoint regions. There is no single bias value $b$ that satisfies the desired conditions for both $R_1$ and $R_2$. This highlights the necessity of adjusting both weights and biases, as mPOLICE does, to achieve the required flexibility for localized control over neuron activations in multiple distinct input domains.

\section{Non-Convex Approximation}
\label{app:nonconvex}

While mPOLICE is fundamentally designed to enforce affine constraints within \emph{convex} polytopes, its principles can be extended to handle \emph{non-convex} regions through approximation. The core strategy involves decomposing the target non-convex region $\mathcal{R}_{\text{nc}}$ into a collection of disjoint (or nearly-disjoint) convex sub-regions $\{\mathcal{R}_{c,1}, \mathcal{R}_{c,2}, \dots, \mathcal{R}_{c,M}\}$ such that their union $\bigcup_{j=1}^M \mathcal{R}_{c,j}$ effectively covers or approximates $\mathcal{R}_{\text{nc}}$. mPOLICE can then be applied to each convex sub-region $\mathcal{R}_{c,j}$ individually, assigning it a unique activation pattern to ensure localized affine behavior $f_{\boldsymbol{\theta}}(\boldsymbol{x}) = \mathbf{\Lambda}_j \boldsymbol{x} + \boldsymbol{\gamma}_j$ for $\boldsymbol{x} \in \mathcal{R}_{c,j}$. If the goal is to enforce a single, consistent affine behavior across the entire non-convex region $\mathcal{R}_{\text{nc}}$, then the constraints imposed on each $\mathcal{R}_{c,j}$ would aim to make $(\mathbf{\Lambda}_j, \boldsymbol{\gamma}_j)$ as similar as possible.

This technique is particularly useful when an overall non-convex region needs to adhere to specific, often simple, behaviors that can be achieved by patching together affine pieces. For instance, in discrete reinforcement learning or path planning, a narrow "corridor" or an irregularly shaped "safe zone" might be approximated by a series of closely-packed convex regions. The "gaps" between these regions, if small enough, might be traversable in a single time step or treated as negligible transition zones where the affine behavior approximately holds.

The effectiveness of this approximation depends on the granularity of the convex decomposition and the proximity of the sub-regions. While it doesn't provide the same provable guarantee over the entire non-convex region as it does for its individual convex components, it offers a practical way to extend mPOLICE's capabilities to more complex geometries. The trade-off involves an increased number of constrained regions (and thus vertices for mPOLICE to handle) if a fine-grained approximation is needed.

In the example shown in Figure \ref{fig:non_convex}, we empirically show that, given a non-convex shape can be decomposed into multiple convex regions, our method allows approximating a non-convex affine region by placing disjoint convex regions in close proximity. The network is trained to learn a saddle background field subject to two affine regions. The left plot shows the two squares are separated by a large gap, so the space between them comprises multiple boundaries (black lines) with no guarantee of affine behavior. In contrast, in the right plot, the squares are placed extremely close to each other (within ten times machine precision), so the region between them effectively becomes part of the boundary separating the two distinct polytopes.

\begin{figure*}[t]
\centering
\includegraphics[width=1.0\linewidth]{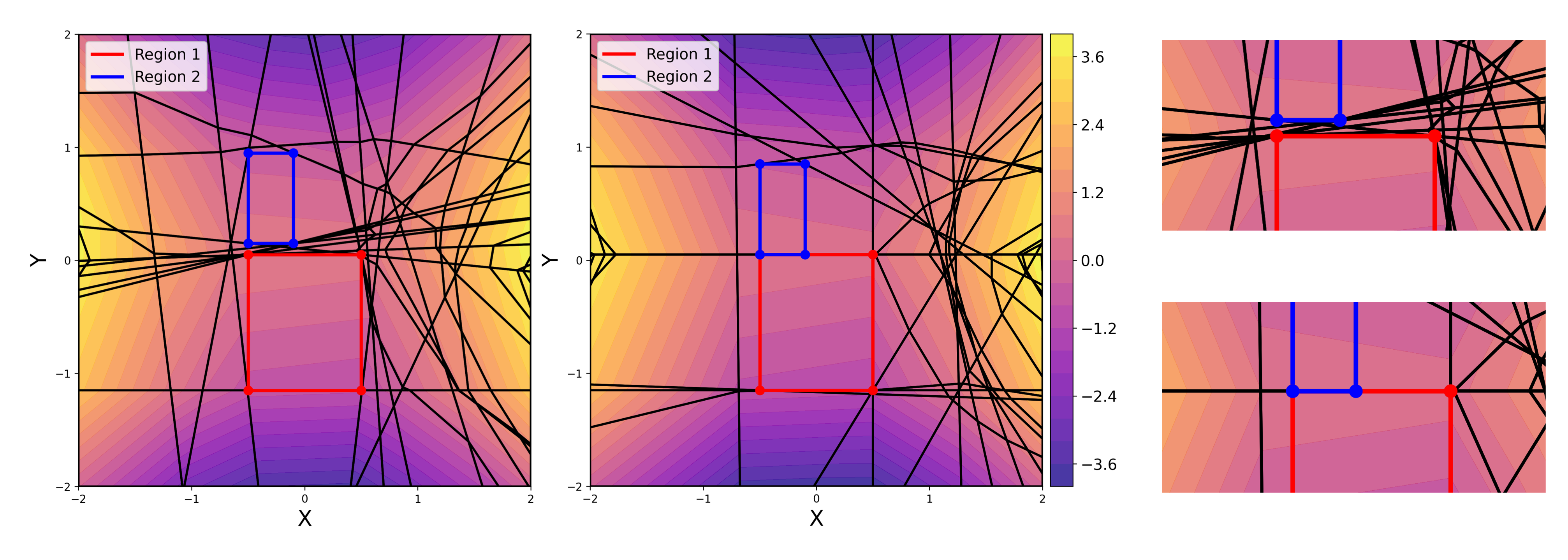}
\caption{Two affine regions approximating a saddle background field. On the left, the large gap between squares spans several polytopes, yielding no affine guarantees between them. On the right, placing the squares very close turns the gap into a shared boundary of two polytopes, effectively approximating a non-convex shape with two convex pieces.}
\label{fig:non_convex}
\end{figure*}

\section{Reinforcement Learning Problem Details}
\label{app:rl_problem_details}

The reinforcement learning environment is a continuous 2D navigation task within a normalized square domain, typically $[-1, 1] \times [-1, 1]$. The agent's objective is to navigate from a starting position to a target location while avoiding two horizontal line obstacles. The agent is trained using the Twin Delayed Deep Deterministic Policy Gradient (TD3) algorithm~\citep{fujimoto2018addressing}.

\paragraph{State Space.} The state $\boldsymbol{s} \in \mathbb{R}^2$ represents the agent's current $(x, y)$ coordinates within the normalized environment.

\paragraph{Action Space.} The action $\boldsymbol{a} \in \mathbb{R}^2$ represents the desired change in position (velocity components $v_x, v_y$). Actions are continuous and typically clipped to a maximum magnitude, e.g., $v_x, v_y \in [-a_{\text{max}}, a_{\text{max}}]$. The agent's position is updated via $\boldsymbol{s}_{t+1} = \boldsymbol{s}_t + \boldsymbol{a}_t \cdot \Delta t$, where $\Delta t$ is a small time step.

\paragraph{Environment Features.}
\begin{itemize}
    \item \textbf{Target}: A fixed point $\boldsymbol{s}_T = (x_T, y_T)$ (e.g., $(0.8, 0.8)$ in our experiments, within the normalized domain).
    \item \textbf{Obstacles}: Two horizontal line segments. For $k \in \{1, 2\}$, obstacle $k$ is defined by its y-coordinate $y_{O_k}$ and an x-interval $[x_{O_k}^{\text{min}}, x_{O_k}^{\text{max}}]$. In our setup, scaled to a normalized $[-1,1]$ domain, example parameters are:
    \begin{itemize}
        \item Obstacle 1: $y_{O_1} = 0.3$, $x \in [-0.3, 0.1]$.
        \item Obstacle 2: $y_{O_2} = -0.3$, $x \in [0.55, 0.95]$.
    \end{itemize}
\end{itemize}

\paragraph{Buffer Regions and Constraints.} For each obstacle $k$, a rectangular buffer region $R_k$ is defined directly below it. If obstacle $k$ spans $x \in [x_{O_k}^{\text{min}}, x_{O_k}^{\text{max}}]$ at height $y_{O_k}$, its buffer region $R_k$ is defined by vertices:
\[
(x_{O_k}^{\text{min}}, y_{O_k} - h_B), (x_{O_k}^{\text{max}}, y_{O_k} - h_B), (x_{O_k}^{\text{max}}, y_{O_k}), (x_{O_k}^{\text{min}}, y_{O_k})
\]
where $h_B$ is the buffer height (e.g., $0.1$ in the normalized environment scale). Within each buffer region $R_k$, we enforce the constraint that the vertical component of the agent's action $a_y$ must be non-positive, i.e., $a_y \le 0$. This forces the agent to move downwards or horizontally, away from the obstacle immediately above it, when inside this buffer zone. This is enforced by ensuring the policy network $f_{\boldsymbol{\theta}}(\boldsymbol{s}) = \boldsymbol{a}$ satisfies $a_y \le 0$ for all $\boldsymbol{s} \in R_k$.

\paragraph{Reward Function.} The reward $r_t$ at each timestep is designed to encourage reaching the target while penalizing undesirable outcomes. It is defined as:
\begin{align*}
    r_t = & -\|\boldsymbol{s}_{t+1} - \boldsymbol{s}_T\|_2 \\
          & + R_{\text{target}} \cdot \mathbb{I}(\|\boldsymbol{s}_{t+1} - \boldsymbol{s}_T\|_2 < \epsilon_T \text{ and } \boldsymbol{s}_{t+1} \notin \bigcup_k R_k) \\
          & + P_{\text{collision}} \cdot \mathbb{I}(\text{agent collided with an obstacle}) \\
          & + P_{\text{bounds}} \cdot \mathbb{I}(\text{agent is out of environment bounds})
\end{align*}
where $\mathbb{I}(\cdot)$ is the indicator function. The components are:
\begin{itemize}
    \item $-\|\boldsymbol{s}_{t+1} - \boldsymbol{s}_T\|_2$ is a shaping reward based on the negative Euclidean distance to the target.
    \item $R_{\text{target}}$ is a positive reward (e.g., $+10$) for reaching the target, awarded if the agent is within a small distance $\epsilon_T$ of $\boldsymbol{s}_T$ and is not inside any buffer region $R_k$.
    \item $P_{\text{collision}}$ is a negative penalty (e.g., $-20$) if the agent's path segment from $\boldsymbol{s}_t$ to $\boldsymbol{s}_{t+1}$ intersects an obstacle.
    \item $P_{\text{bounds}}$ is a negative penalty (e.g., $-10$) if the agent moves outside the predefined environment boundaries.
\end{itemize}
The episode terminates upon reaching the target, a collision, or exceeding a maximum episode length.

The hyperparameters used in the reinforcement learning experiment (Section~\ref{subsec:rl}) are detailed in Table~\ref{tab:rl_hyperparameters}.

\begin{table}[htbp]
\footnotesize
\setlength{\tabcolsep}{6pt}
\renewcommand{\arraystretch}{0.9}
\centering
\caption{Hyperparameters for the Reinforcement Learning Experiment.}
\label{tab:rl_hyperparameters}
\begin{tabular}{@{}lll@{}}
\toprule
\textbf{Category} & \textbf{Parameter} & \textbf{Value} \\
\midrule
\multirow{10}{*}{Environment} 
& State dimensionality & 2 \\
& Action dimensionality & 2 \\
& Max action magnitude & 1.0 \\
& Time step ($\Delta t$) & 0.2 \\
& Max episode length & 500 \\
& Buffer height (norm.) & 0.1 \\
& Out-of-bounds penalty & -10.0 \\
& Collision penalty & -20.0 \\
& Target reached bonus & 10.0 \\
& Target threshold & 0.1 \\
\midrule
\multirow{5}{*}{Initial Sampling (\%)} 
& Near target & 10 \\
& Below buffer & 20 \\
& Above buffer & 20 \\
& Random & 50 \\
& Target start radius & 0.5 \\
\midrule
\multirow{13}{*}{TD3 Algorithm} 
& Replay buffer size & 2,000,000 \\
& Batch size & 20,000 \\
& Discount ($\gamma$) & 0.99 \\
& Target update rate ($\tau$) & 0.005 \\
& Policy noise & 0.2 \\
& Noise clip & 0.5 \\
& Policy update freq. & 2 \\
& Actor learning rate & $3 \times 10^{-4}$ \\
& Critic learning rate & $3 \times 10^{-4}$ \\
& Exploration noise & 0.1 \\
& Random exploration steps & 1,000 \\
& Training timesteps & 20,000 \\
\midrule
\multirow{5}{*}{Architecture} 
& Actor width & 64 \\
& Actor num. hidden layers & 3 \\
& Critic width & 64 \\
& Critic num. hidden layers & 3 \\
& Activation function & Leaky ReLU (0.01) \\
\midrule
\multirow{3}{*}{mPOLICE} 
& Start timestep & 16,000 \\
& Interval (timesteps) & 1,000 \\
& Margin ($\delta$) & 0.0 \\
\midrule
\multirow{4}{*}{Final Loop} 
& Max iterations & 50 \\
& Steps per iteration & 1,000 \\
& Patience & 50 \\
& Penalty weight & 1.0 \\
\bottomrule
\end{tabular}
\end{table}

\section{Fluid Dynamics Experiment Details}
\label{app:fluid_dynamics_details}

This supplementary material provides further details on the fluid dynamics experiment discussed in Section~\ref{subsec:fluid_dynamics}.

\paragraph{Problem Setup.}
The task is to predict the scalar velocity magnitude at interior points of a 2D grid representing a fluid flow. The overall grid dimensions are $400 \times 100$ pixels, with the prediction focused on the interior $398 \times 98$ grid. Input to the network consists of normalized $(x,y)$ coordinates of these interior grid points. The ground truth velocity magnitude is derived from a pre-computed simulation data file.

Two square regions, each $16 \times 16$ pixels, are defined within the flow field. The centers of these squares are located at approximately $(80, 33)$ and $(80, 66)$ in the grid coordinates. Within these two regions, the velocity magnitude is constrained to be zero. This constraint simulates the presence of impenetrable obstacles or zones where the fluid is known to be static. The mPOLICE algorithm is employed to enforce this zero-velocity condition strictly.

\paragraph{Network Architecture and Training Hyperparameters.}
The neural network used is an \texttt{AFFIRMNetwork} with 2 input dimensions (scaled x, y coordinates), a hidden dimension of 128, and 3 hidden layers, using Leaky ReLU activations.

The training process consists of two phases:
\begin{itemize}
    \item \textbf{Initial Training (Baseline Model):} The network is first trained for 500 epochs on the entire dataset without explicit constraint enforcement. The learning rate is $0.001$, and a batch size of 256 is used with an Adam optimizer and MSE loss.
    \item \textbf{mPOLICE Fine-tuning (Constrained Model):} The baseline model's weights are then used as a starting point for fine-tuning with mPOLICE. Key hyperparameters for the mPOLICE enforcement and fine-tuning loop include:
    \begin{itemize}
        \item Violation Threshold ($\epsilon_{\text{target}}$): $1.0 \times 10^{-4}$. This is the target for the maximum allowable constraint violation in unscaled output space.
        \item Sign Assignment Method: Mean-based (average pre-activation over region vertices).
        \item Optimizer for fine-tuning: Adam with a learning rate of $1.0 \times 10^{-4}$ (0.1 times the initial learning rate).
        \item Initial Constraint Penalty Weight ($\lambda_{\text{constraint}}$): $1.0$.
        \item Maximum Constraint Penalty Weight: $100.0$.
        \item Penalty Multiplier: $1.5$.
        \item Minimum Fine-tuning Epochs (per outer iteration): 30.
        \item Maximum Fine-tuning Epochs (per outer iteration): 50.
        \item Patience (for early stopping within fine-tuning): 20 epochs.
        \item Enforcement Margin ($\delta$ for QP solver): $0.0$.
    \end{itemize}
\end{itemize}

\paragraph{Results Visualization.}
Figure~\ref{fig:fluid_flow_comparison} shows a comparison of the ground truth velocity field, the baseline model's prediction, and the mPOLICE constrained model's prediction across the entire interior domain. Figure~\ref{fig:fluid_flow_zoomed} provides zoomed-in views of the two constrained square regions for the same three fields, allowing for a closer inspection of how the zero-velocity constraint is met by the mPOLICE model.

\begin{figure}[ht]
    \centering
    \includegraphics[width=\textwidth]{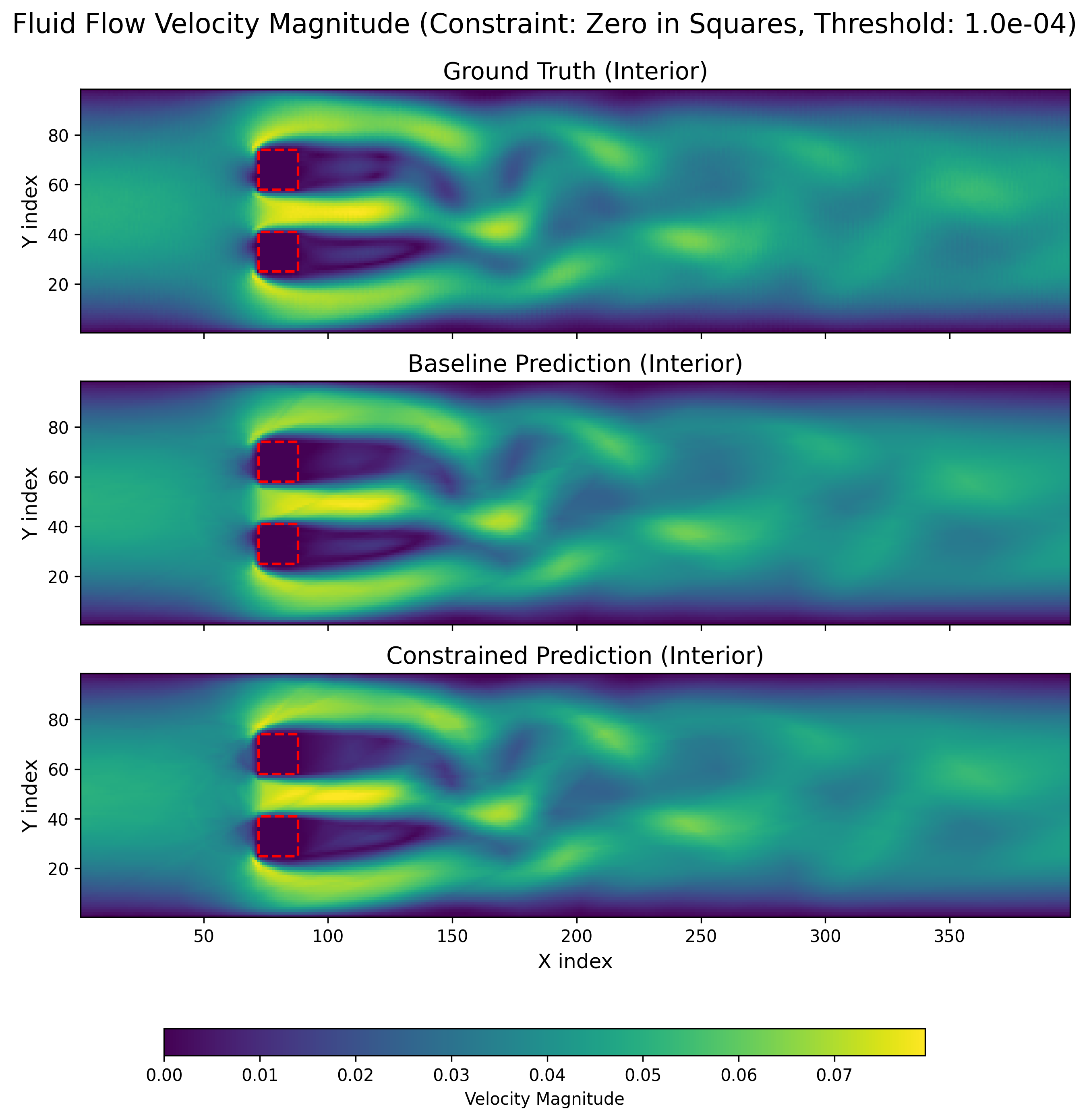}
    \caption{Comparison of fluid flow velocity magnitude fields: (Top) Ground Truth, (Middle) Baseline Prediction, and (Bottom) mPOLICE Constrained Prediction. The red dashed squares indicate the regions where zero velocity is enforced by mPOLICE. The color bar indicates velocity magnitude. Data shown is for the interior of the domain.}
    \label{fig:fluid_flow_comparison}
\end{figure}

\begin{figure}[]
    \centering
    \includegraphics[width=0.9\textwidth]{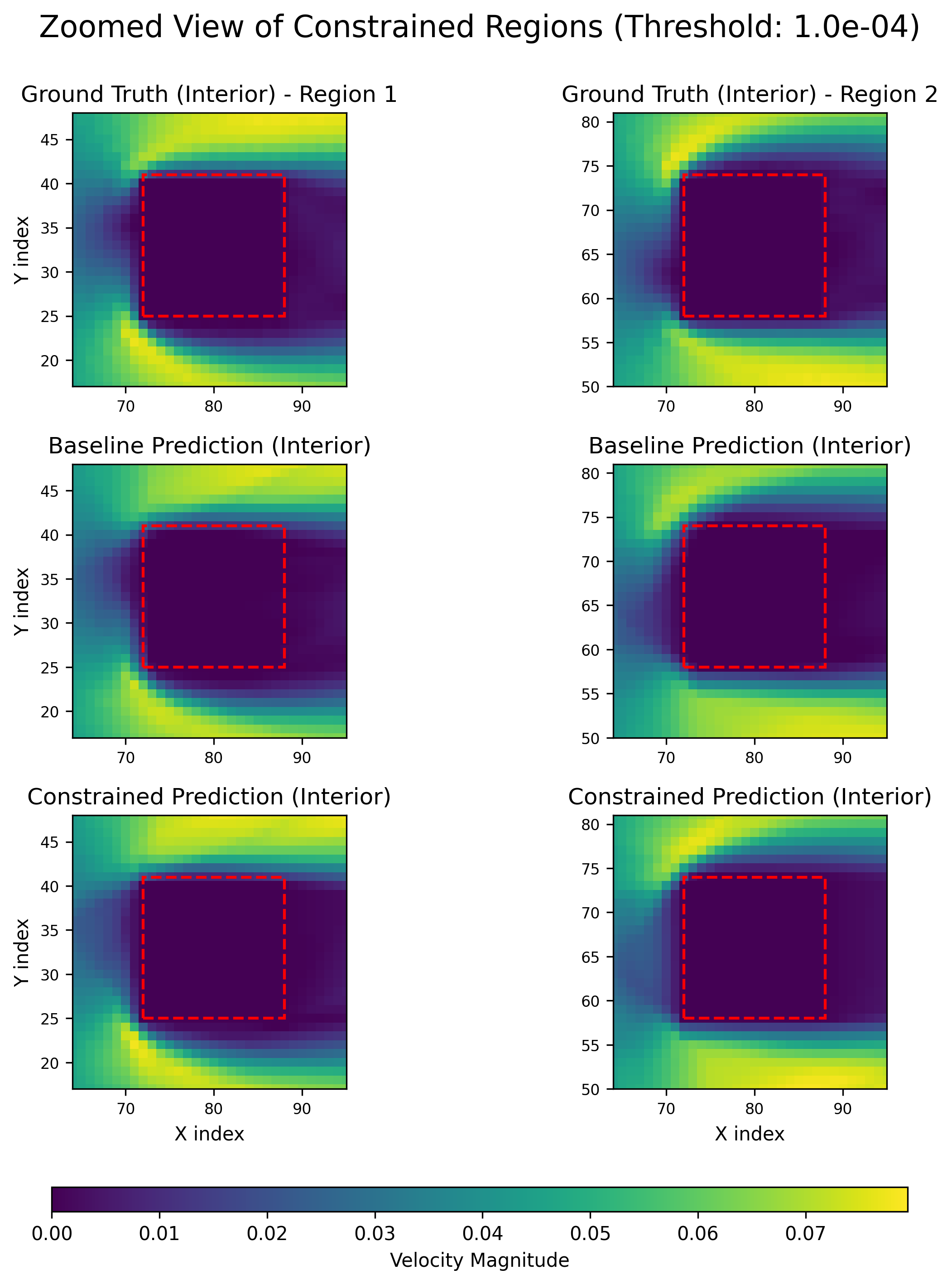}
    \caption{Zoomed-in views of the two constrained square regions for the fluid flow velocity magnitude. Each row corresponds to: (Top) Ground Truth, (Middle) Baseline Prediction, and (Bottom) mPOLICE Constrained Prediction. Left column shows Region 1 (centered near y=33), and right column shows Region 2 (centered near y=66). The red dashed squares delineate the $16 \times 16$ constraint areas. The mPOLICE predictions demonstrate adherence to the zero-velocity constraint within these squares, appearing as dark blue/purple matching the zero-point on the color scale.}
    \label{fig:fluid_flow_zoomed}
\end{figure}

\section{Scalability and Benchmarking}
\label{app:scalability_benchmarking}

The computational cost of mPOLICE during training is primarily driven by two components: assigning sign patterns to regions ($T_{\text{Assign}}$) and enforcing these patterns by adjusting network parameters ($T_{\text{Enforce}}$). The sign assignment step involves forward passes for vertices in each region and is generally very fast. The enforcement step, however, requires solving a quadratic program (QP) for each neuron in each layer whose parameters are being adjusted. This can be more computationally intensive. It's important to note that mPOLICE incurs \textbf{zero overhead during inference}, as all adjustments are made to the network parameters offline.

Our current implementation of the enforcement step is not highly optimized. For instance, the QPs are solved sequentially for each neuron on the CPU, which involves per-neuron overhead and potential communication costs if the main model resides on a GPU. Significant speedups could likely be achieved by parallelizing these QP solves (e.g., batching them or using specialized hardware if available) and by optimizing data transfers.

We conducted benchmarking experiments to assess how $T_{\text{Enforce}}$ scales with network architecture and the complexity of the constrained regions. Figure~\ref{fig:mpolice_benchmark_results} shows the enforcement time as a function of the number of hidden layers, for different network widths (number of neurons per layer) and total number of vertices defining the constrained regions. The parameters varied include:
\begin{itemize}
    \item \textbf{Number of Hidden Layers}: The x-axis represents the number of hidden layers whose parameters are subject to enforcement (e.g., 1, 2, or 3 hidden layers in our \texttt{AFFIRMNetwork} architecture).
    \item \textbf{Net\_Width}: Line colors differentiate network widths (64, 128, or 256 neurons per hidden layer).
    \item \textbf{N\_Regions (Total\_Vertices)}: Marker styles differentiate configurations based on the number of constrained regions, which directly correlates with the total number of vertices involved (e.g., 2 regions with 16 vertices each for 32 total vertices, up to 8 regions with 16 vertices each for 128 total vertices).
\end{itemize}

As observed in Figure~\ref{fig:mpolice_benchmark_results} and Table \ref{tab:timing_table}, $T_{\text{Enforce}}$ increases with:
\begin{itemize}
    \item \textbf{Network Depth (Number of Hidden Layers)}: More layers mean more sets of QPs to solve.
    \item \textbf{Network Width}: Wider layers have more neurons, each requiring a QP solve.
    \item \textbf{Total Number of Vertices}: This is a critical factor. Each QP's constraint set size depends on the total number of vertices across all regions that the neuron's sign pattern must satisfy. The figure clearly shows that configurations with more total vertices (e.g., the black square markers for 8 regions / 128 total vertices) exhibit the longest enforcement times, especially for wider and deeper networks. The enforcement time appears to scale significantly with the number of vertices, as this directly impacts the complexity of each QP.
\end{itemize}

While the scaling with the total number of vertices might seem concerning, it is often manageable in practice for several reasons:
\begin{enumerate}
    \item \textbf{Low-Dimensional Constraints}: Many practical applications involve constraints defined on relatively low-dimensional input subspaces or simple geometric shapes (lines, planes, hyper-rectangles, spheres). These shapes can often be adequately represented by a small number of vertices. For instance, a hyper-rectangle in $D$ dimensions has $2^D$ vertices, which is manageable for small $D$.
    \item \textbf{Focus of Constraints}: Constraints are often applied to specific, critical regions rather than exhaustively throughout a high-dimensional space.
    \item \textbf{No Inference Overhead}: The enforcement cost is a one-time (or periodic during fine-tuning) training cost. Once the network is trained, inference is standard and fast.
    \item \textbf{Optimization Potential}: As mentioned, our current implementation has significant room for optimization, which could substantially reduce these enforcement times.
\end{enumerate}
The detailed tabular results in Table \ref{tab:timing_table} further break down $T_{\text{Assign}}$ and $T_{\text{Enforce}}$. Absolute enforcement times are reported in seconds, as presenting them as a percentage of baseline training time can be misleading due to the high variability of baseline training durations across different tasks and datasets. 
$T_{\text{Assign}}$ remains negligible across all tested configurations, typically in the milliseconds. $T_{\text{Enforce}}$ can range from milliseconds for small networks and few vertices to several minutes for the largest configurations tested (e.g., 8 regions/128 vertices, 256 width, 3 hidden layers took $\sim$217 seconds). This highlights that while mPOLICE is provably effective, careful consideration of the number of vertices and network size is needed for computationally intensive scenarios, and further implementation optimizations would be beneficial.

\begin{table}[htbp]
    \centering
    \caption{Timing results (seconds) for mPOLICE's sign assignment ($T_{\text{Assign}}$) and sign enforcement ($T_{\text{Enforce}}$) procedures across various network configurations and region complexities. `Num\_Hidden\_L' refers to the number of hidden layers whose parameters are adjusted during enforcement.}
    \label{tab:timing_table}
\begin{tabular}{llllll}
\toprule
   N\_Regions &   Total\_Vertices &   Net\_Width &   Num\_Hidden\_L &   T\_Assign (s) &   T\_Enforce (s) \\
\midrule
            2 &                32 &           64 &                1 &          0.0036 &           0.0358 \\
            2 &                32 &           64 &                2 &          0.002  &           0.1913 \\
            2 &                32 &           64 &                3 &          0.0028 &           0.3057 \\
            2 &                32 &          128 &                1 &          0.0018 &           0.0605 \\
            2 &                32 &          128 &                2 &          0.0035 &           0.8552 \\
            2 &                32 &          128 &                3 &          0.0051 &           1.7041 \\
            2 &                32 &          256 &                1 &          0.0034 &           0.1651 \\
            2 &                32 &          256 &                2 &          0.0065 &           6.0467 \\
            2 &                32 &          256 &                3 &          0.0097 &          11.8718 \\
            4 &                64 &           64 &                1 &          0.0021 &           0.0347 \\
            4 &                64 &           64 &                2 &          0.0037 &           0.6893 \\
            4 &                64 &           64 &                3 &          0.0055 &           1.3718 \\
            4 &                64 &          128 &                1 &          0.0035 &           0.1008 \\
            4 &                64 &          128 &                2 &          0.0068 &           2.5825 \\
            4 &                64 &          128 &                3 &          0.0101 &           6.0731 \\
            4 &                64 &          256 &                1 &          0.0066 &           0.2108 \\
            4 &                64 &          256 &                2 &          0.0684 &          19.7345 \\
            4 &                64 &          256 &                3 &          0.0193 &          32.5451 \\
            8 &               128 &           64 &                1 &          0.0041 &           0.1349 \\
            8 &               128 &           64 &                2 &          0.0076 &           0.7988 \\
            8 &               128 &           64 &                3 &          0.011  &           1.8662 \\
            8 &               128 &          128 &                1 &          0.0071 &           0.316  \\
            8 &               128 &          128 &                2 &          0.0137 &           9.956  \\
            8 &               128 &          128 &                3 &          0.0202 &          29.6049 \\
            8 &               128 &          256 &                1 &          0.0132 &           1.1019 \\
            8 &               128 &          256 &                2 &          0.0258 &          59.4884 \\
            8 &               128 &          256 &                3 &          0.0385 &         217.526  \\
\bottomrule
\end{tabular}
\end{table}

\begin{figure}[ht]
    \centering
    \includegraphics[width=1.0\textwidth]{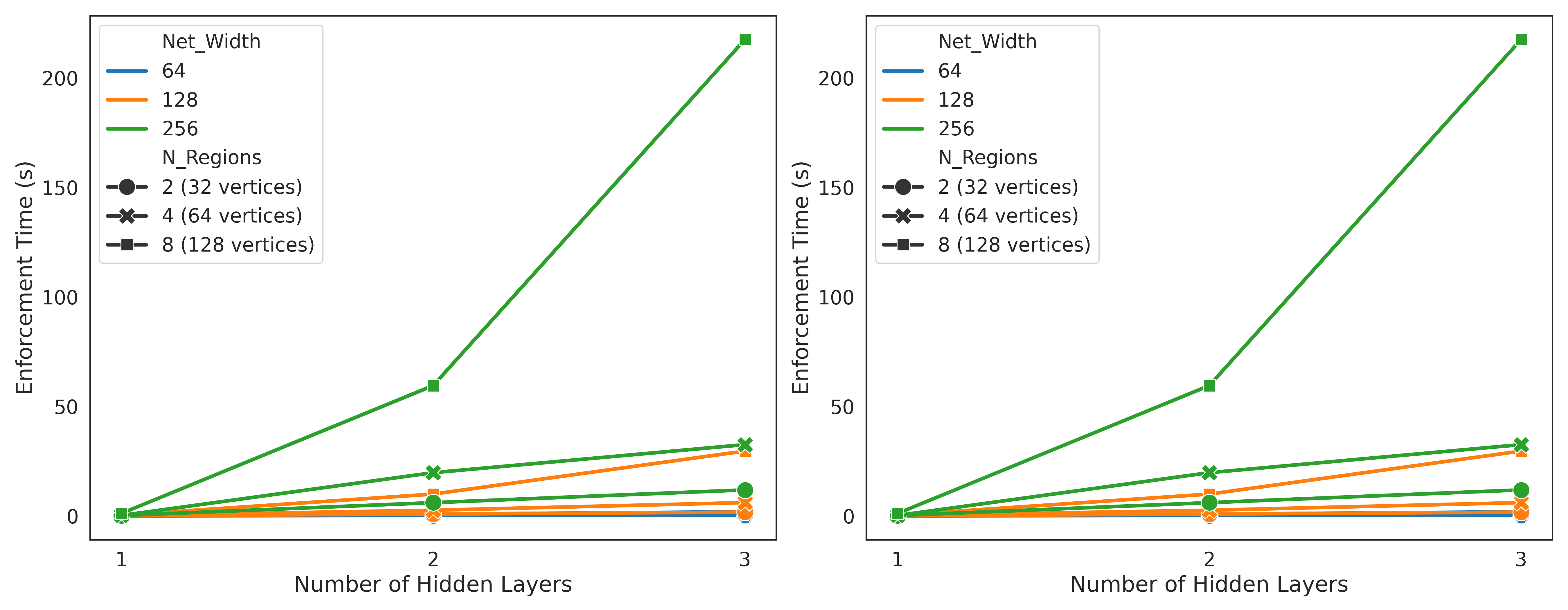}
    \caption{mPOLICE enforcement time ($T_{\text{Enforce}}$) as a function of the number of hidden layers subject to enforcement. Different line colors represent network widths (neurons per layer), and different marker styles represent the number of constrained regions (and thus total vertices). Enforcement time increases with network depth, width, and, most notably, with the total number of vertices being constrained.}
    \label{fig:mpolice_benchmark_results}
\end{figure}

\section{Illustrative Classification and Regression Examples}
\label{app:illustrative_examples_improved}

To further clarify the mechanics and effects of mPOLICE, this section presents simple, illustrative examples in standard classification and regression tasks. These low-dimensional problems are chosen for ease of visualization and to clearly demonstrate how mPOLICE enforces localized affine behavior and manages constraints in multiple regions.

\subsection{Illustrative Classification Task}
We demonstrate mPOLICE on a two-dimensional spiral classification task. The goal is to correctly classify points belonging to two intertwined spirals. To showcase multi-region constraint enforcement, we define two disjoint square regions along the arms of the spirals. Within each of these squares, we require the network's output logit to behave as an affine function. This is achieved by assigning a unique activation pattern to each square (using majority voting as described in Section~\ref{subsec:assigning_unique_patterns}) and then using mPOLICE's enforcement mechanisms. For this illustrative example, a simple network architecture consisting of a single hidden layer with 32 neurons and a linear output layer is used, trained with Binary Cross Entropy (BCE) loss.

The training process, depicted in Figure~\ref{fig:spiral}, begins with 30 epochs of unconstrained training to allow the network to learn the general decision boundary. Subsequently, mPOLICE's sign pattern and parameter adjustment steps are applied after each epoch. As shown in the loss curve (Figure~\ref{fig:spiral}, top right), enforcing the affine constraints initially causes a spike in the task loss. This occurs because the network parameters are abruptly changed to satisfy the hard affine constraints, potentially disrupting the previously learned function. However, the loss then steadily declines as the network adapts, balancing task performance with strict constraint adherence.

The polytope partitions induced by the network's ReLU activations are visualized using techniques from \citep{humayun2022splinecam,humayun2022exact} (Figure~\ref{fig:spiral}, left and bottom right). Before mPOLICE's fine-tuning begins (e.g., at epoch 30), multiple polytope boundaries (white lines) can be seen passing through the constrained square regions. After mPOLICE enforcement, each square region is contained entirely within a single, distinct affine polytope, with no internal polytope boundaries crossing them. The inset in Figure~\ref{fig:spiral} (bottom right, highlighted by a green dashed boundary) clearly shows one such constrained square residing in its own polytope, demonstrating localized affine behavior while the network remains non-linear elsewhere.

\begin{figure}[t]
    \centering
    \includegraphics[width=\textwidth]{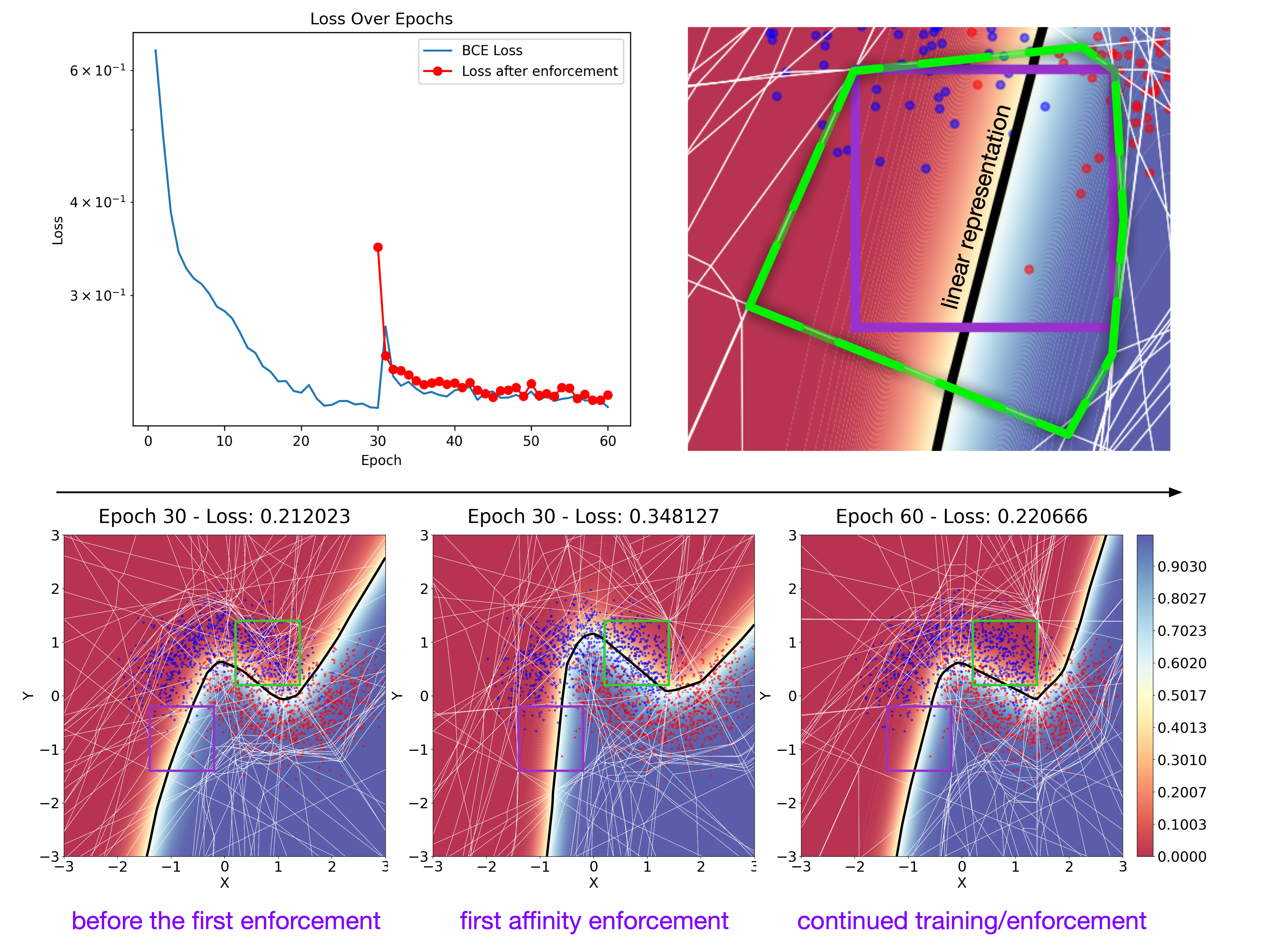}
    \caption{Multi-region constraint enforcement on a spiral classification task, with the loss curve initially spiking at epoch 30 after the first enforcement step but steadily converging over subsequent epochs. The violet square's convex polytope is highlighted with a green dashed boundary, illustrating how the network enforces affine behavior locally.}
    \label{fig:spiral}
\end{figure}

\subsection{Illustrative Regression Task}
For an illustrative regression example, a neural network is trained to approximate the function $f(x) = \sin(x)$ over the interval $[0, 2\pi]$. Two disjoint intervals are defined as constrained regions:
\begin{itemize}
    \item Region $R_1 = [\pi/3, 3\pi/4]$: The network output $f_{\boldsymbol{\theta}}(x)$ is constrained to be exactly equal to $\sin(\pi/3)$ (an equality constraint).
    \item Region $R_2 = [\pi + \pi/3, \pi + 3\pi/4]$: The network output $f_{\boldsymbol{\theta}}(x)$ is constrained to be less than or equal to $-0.5$ (an inequality constraint).
\end{itemize}
mPOLICE is employed by assigning unique activation patterns to $R_1$ and $R_2$. The network's weights and biases are then adjusted to ensure affine behavior within each region, such that the constraints are satisfied at their vertices. The training process follows Algorithm~\ref{alg:fine_tuning_mpolice}: after an initial "warm-up" phase of 2048 epochs where the network learns the general $\sin(x)$ trend (with constraint penalty $\lambda=0$), a fine-tuning phase begins. During fine-tuning, a composite loss $\mathcal{L}_{\text{total}} = \mathcal{L}_{\text{data}} + \lambda_{\text{constraint}}\mathcal{L}_{\text{constraint}}$ is minimized, where $\mathcal{L}_{\text{data}}$ is the Mean Squared Error (MSE) on unconstrained data, and $\mathcal{L}_{\text{constraint}}$ penalizes constraint violations at the vertices of $R_1$ and $R_2$. The fine-tuning aims to reduce the constraint violation below a predefined tolerance of $5 \times 10^{-4}$.

The results are presented in Figure~\ref{fig:constrained_sin_supplementary material} and Table~\ref{tab:constraint_time_comparison_supplementary material}. The mPOLICE-constrained network successfully adheres to the specified constraints: its output is constant at $\sin(\pi/3)$ within $R_1$ and remains at or below $-0.5$ within $R_2$. As shown in Table~\ref{tab:constraint_time_comparison_supplementary material}, the constraint violation for the enforced model drops to $0.000432$, satisfying the target threshold of $5 \times 10^{-4}$. This precision comes at the cost of a slight increase in MSE on the data outside the constrained regions (from $0.001895$ for the baseline to $0.002557$ for the constrained model), illustrating the typical trade-off. The mPOLICE fine-tuning and enforcement step constituted approximately $41\%$ of the total training time for the constrained model in this example. These results clearly demonstrate mPOLICE's ability to enforce distinct types of affine constraints in multiple regions with provable guarantees.

\begin{figure}[t]
\centering
\includegraphics[width=1.0\linewidth]{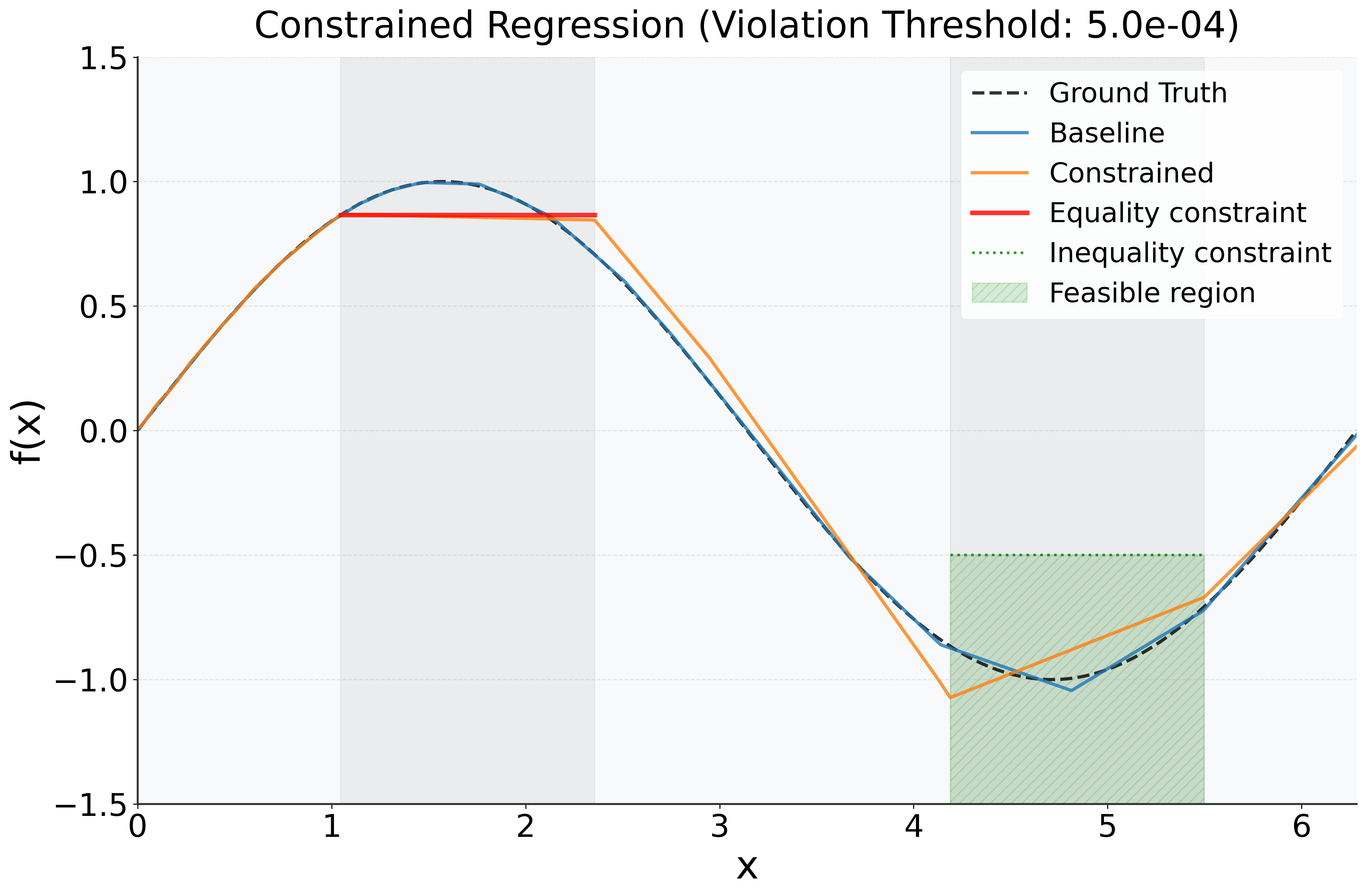}
\caption{Neural network approximation of constrained $\sin(x)$ with enforced affine constraints. This example illustrates the constraint enforcement on a regression task. The red line in $R_1$ shows the equality target, and the green line/shaded area in $R_2$ shows the inequality constraint and feasible region.}
\label{fig:constrained_sin_supplementary material}
\end{figure}

\begin{table}[ht]
    \centering
    \caption{Comparison of models based on Mean Squared Error (MSE) on data outside constrained regions, constraint violation, and runtime for the illustrative regression example. The violation threshold for mPOLICE fine-tuning was $5 \times 10^{-4}$.}
    \begin{tabular}{lccc}
        \hline
        Step      & MSE (excluding constraints)      & Violation  & Time (s) \\
        \hline
        Baseline    & 0.001895 & 0.001895   & 19.23    \\
        Enforcement & 0.002557 & 0.000432   &   13.57    \\
        \hline
    \end{tabular}
    \label{tab:constraint_time_comparison_supplementary material}
\end{table}

\end{document}